\documentclass{article} 
\usepackage{iclr2020_conference,times}
\iclrfinalcopy

\usepackage{url}
\usepackage{tikz}
\usetikzlibrary{bayesnet}

\usepackage[utf8]{inputenc} 
\usepackage[T1]{fontenc}    
\usepackage{hyperref}      
\usepackage{url}            
\usepackage{booktabs}       
\usepackage{amsfonts}       
\usepackage{nicefrac}       
\usepackage{microtype}      
\usepackage{natbib}

\usepackage{amsmath,amsfonts,bm,amssymb}
\usepackage{amsthm}
\usepackage{multirow}
\usepackage{multicol}
\usepackage{caption}
\usepackage{subcaption}
\usepackage{wrapfig}
\usepackage{lipsum}
\usepackage{makecell}
\usepackage{color,xcolor}
\usepackage{thm-restate}

\def\eqref#1{equation~\ref{#1}}

\def\1{\bm{1}}

\def\vtheta{{\bm{\theta}}}

\def\vx{{\bm{x}}}

\def\vz{{\bm{z}}}

\def\vtheta{{\bm{\theta}}}
\def\vphi{{\bm{\phi}}}

\def\mI{{\bm{I}}}

\def\mX{{\bm{X}}}

\DeclareMathAlphabet{\mathsfit}{\encodingdefault}{\sfdefault}{m}{sl}
\SetMathAlphabet{\mathsfit}{bold}{\encodingdefault}{\sfdefault}{bx}{n}

\def\gL{{\mathcal{L}}}

\def\gN{{\mathcal{N}}}

\newcommand{\KL}{D_{\mathrm{KL}}}
\newcommand{\Var}{\mathrm{Var}}

\DeclareMathOperator*{\argmin}{arg\,min}

\newcommand{\bb}[1]{{\mathbb{#1}}}

\title{Unsupervised Out-of-Distribution Detection \\ with Batch Normalization}

\author{Jiaming Song \& Yang Song \& Stefano Ermon \\
Stanford University\\
\texttt{\{tsong, yangsong, ermon\}@cs.stanford.edu} 
}

\begin{document}

\maketitle

\begin{abstract}
Likelihood from a generative model is a natural statistic for detecting out-of-distribution (OoD) samples. However, generative models have been shown to assign higher likelihood to OoD samples compared to ones from the training distribution, preventing simple threshold-based detection rules. We demonstrate that OoD detection fails even when using more sophisticated statistics based on the likelihoods of individual samples. To address these issues, we propose a new method that leverages batch normalization. We argue that batch normalization for generative models challenges the traditional \emph{i.i.d.} data assumption and changes the corresponding maximum likelihood objective. Based on this insight, we propose to exploit in-batch dependencies for OoD detection.
Empirical results suggest that this leads to more robust detection for high-dimensional images.
\end{abstract}

\section{Introduction}

Modern neural network models are known to make poorly calibrated predictions~\citep{guo2017calibration,kuleshov2018accurate}, and can be highly confident even 
for unrecognizable or irrelevant inputs~\citep{nguyen2015deep,moosavi2017universal}. This has serious implications for AI safety~\citep{amodei2016faulty} in real world deployments, where a model could receive inputs that are beyond its training distribution.
Detecting examples that are out of the training distribution becomes a viable solution: when encountering such samples, the model could choose to provide low confidence estimates or even abstain from making predictions~\citep{cortes2017online}.

Density estimation is one approach to detecting out-of-distribution (OoD) samples. A likelihood-based model is trained on the input samples; during evaluation, samples that have low likelihoods are assumed to be out-of-distribution. For high-dimensional inputs (such as images), deep generative models have been able to generate realistic samples as well as achieving good compression capabilities, which indicates high likelihoods on the training distribution~\citep{balle2016end,kingma2013auto,kingma2018glow,oord2016pixel}; 
thus, recent works have considered using deep generative models to detect out-of-distribution samples~\citep{li2018anomaly}. However, contrary to popular belief, density estimates by deep generative models are highly inaccurate~\citep{nalisnick2018do}. For example, a Glow~\citep{kingma2018glow} model trained on CIFAR10 gives higher likelihood estimates to SVHN samples than CIFAR10 ones, which makes accurate OoD detection impossible. While alternative statistics based on likelihood estimates have been proposed to alleviate this issue~\citep{choi2018waic,song2017pixeldefend}, they are not able to detect OoD samples consistently.

In this paper, we propose a new statistic to detect OoD samples using deep generative models trained with batch normalization~\citep{ioffe2015batch}. We argue that using batch normalization not only improves optimization (as argued in~\citep{kohler2018towards,santurkar2018how}), but also challenges the i.i.d. assumption underlying typical  likelihood-based objectives.
We show in that the training objective of generative models with batch normalization can be interpreted as maximum 
pseudo-likelihood over a different joint distribution that does not assume data in the same batch are i.i.d. samples. 
Empirically, we demonstrate that \emph{over this joint distribution, the estimated likelihood of a batch of OoD samples is much lower than that of in-distribution samples}. 
This allows us to propose a permutation test which outperforms existing methods by a significant margin without modifying how the underlying generative model is trained. In particular, we achieve near-perfect performance even on cases such as Fashion MNIST vs. KMNIST, where the likelihood distributions for single samples are nearly identical (see Figure~\ref{fig:demo}, left). While generative models trained with BatchNorm might not provide state-of-the-art likelihood numbers on the test set, it remains competitive and could be more suited for the OoD detection problem~\citep{kingma2018glow,nalisnick2018do}.

\section{Background}
\label{sec:background}
\paragraph{Setup} Let $\mX = \{\vx_i\}_{i=1}^{N}$ be a set of $N$ observations; we assume that ``in-distribution'' samples are drawn i.i.d. from some underlying distribution $\vx \sim p(\vx)$. 
Our goal is to detect whether a test sample $\vx$ is ``out-of-distribution''; we consider drawing the test samples from an unknown distribution $q(\vx)$, but we do not have access to samples from $q(\vx)$ during training. 
Following definitions in ~\citep{hendrycks2016a,liang2017enhancing}, we define OoD examples as the test samples that have low densities under $p(\vx)$. We note that samples from $q(\vx)$ are not necessarily all OoD; for example, if $q = p$ then most samples should not be classified as OoD.
Detection of OoD samples 
is useful for applications such as anomaly detection~\citep{li2018anomaly} and exploration in reinforcement learning~\citep{ostrovski2017count}. 



\begin{figure}[t]
\centering
\begin{minipage}{0.35\textwidth}
    \centering
    \includegraphics[width=\textwidth]{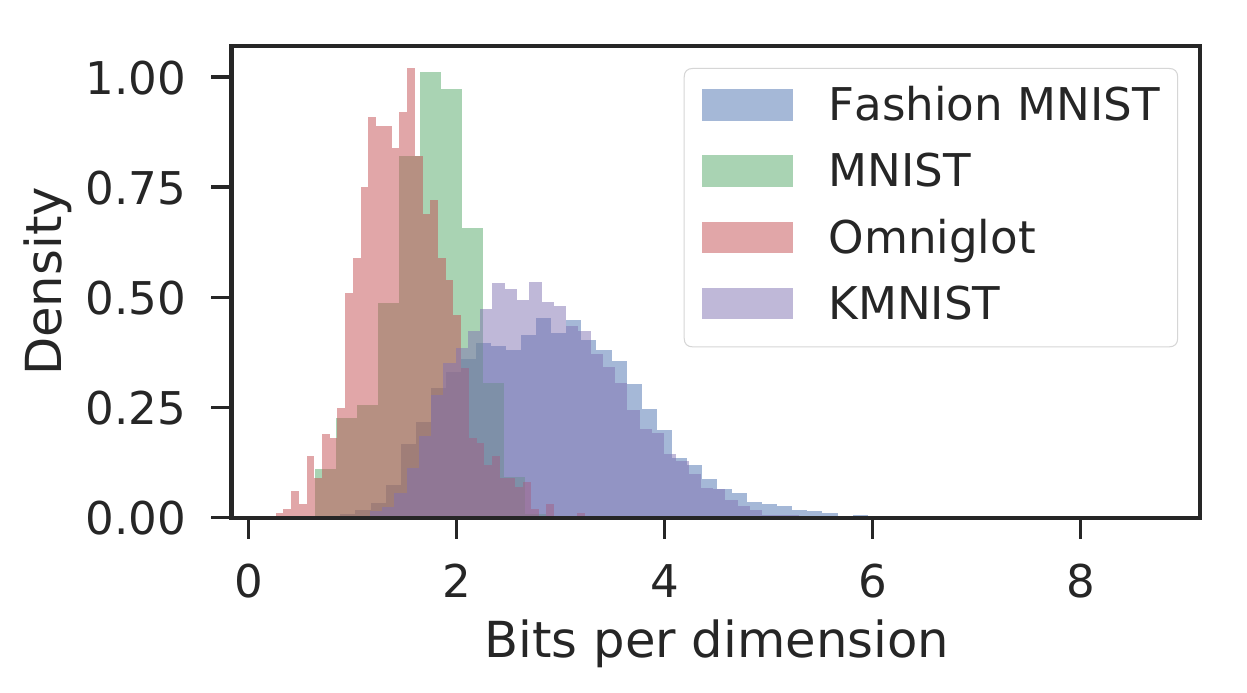}
\end{minipage}
~
\begin{minipage}{0.35\textwidth}
    \centering
    \includegraphics[width=\textwidth]{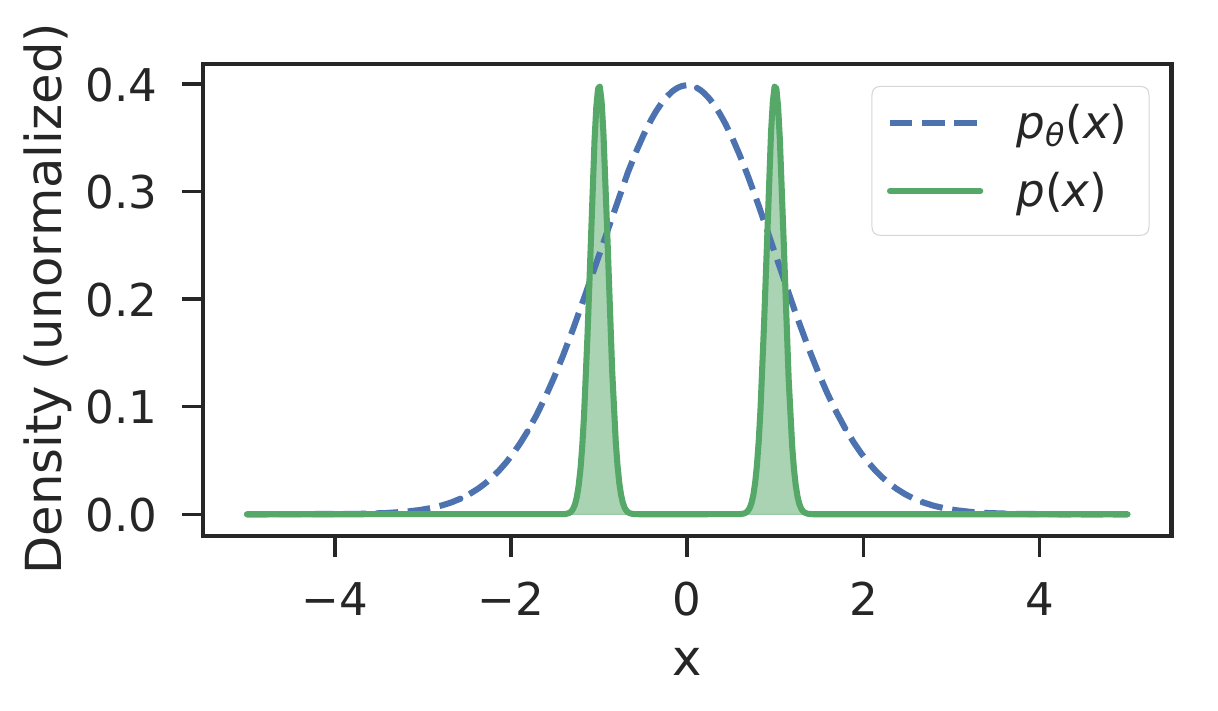}
\end{minipage}
\caption{(Left) Density estimation with a RealNVP model trained on Fashion MNIST. The model assigns similar / higher likelihoods to several OoD datasets. (Right) Model mis-specification can result in OoD samples having higher log-likelihoods.}
\label{fig:demo}
\end{figure}

\paragraph{Likelihood-based generative models for out-of-distribution detection}
Likelihood-based 
models are trained by finding parameters $\vtheta$ such that the corresponding $p_\vtheta(\vx)$ is as close as possible to the 
training distribution $p(\vx)$,
in terms of Kullback-Leibler divergence (KL):
\begin{gather*}
    \argmin_\vtheta \KL(p(\vx) \Vert p_\vtheta(\vx)) = \argmin_\vtheta \bb{E}_{p(\vx)}\left[\log p(\vx) - \log p_\vtheta(\vx)\right] 
\end{gather*}

Finding appropriate solutions depends heavily on choosing the right family of (parametrized) probability distributions.
In the context of images, deep neural networks have demonstrated good 
performance.  Likelihood-based deep generative models include variational autoencoders \citep{kingma2013auto} (where variational lower bounds are optimized), autoregressive models~\citep{oord2016pixel,salimans2017pixelcnn}, and invertible flow models~\citep{dinh2016density,kingma2018glow} (where exact likelihood is optimized).

Given a sample $\vx$, we could attempt to test whether it is OoD or not 
by evaluating its likelihood $-\log p_\vtheta(\vx)$.
Intuitively, if $p_\theta(\vx)$ is sufficiently similar to $p(\vx)$, OoD samples should have low likelihood under $p_\theta(\vx)$.
However, this is not the case for some deep generative models~\citep{nalisnick2018do}, where 
OoD samples 
are observed to have higher likelihoods than samples from the training distribution $p(\vx)$ across several datasets and models. For example, state-of-the-art models trained on CIFAR10~\citep{krizhevsky2009learning} assign higher likelihood to SVHN~\citep{netzer2011reading} samples than CIFAR10 samples. 
A similar phenomenon has also been observed by~\citet{nalisnick2018do} for models trained on Fashion MNIST~\citep{xiao2017fashion} (Figure~\ref{fig:demo} (left)), where images with lower variance across pixels (such as MNIST) are assigned higher likelihoods. 
This invalidates the assumption that OoD samples are assigned lower likelihoods by existing deep generative models.


To alleviate this issue, \citet{choi2018waic} use 
Watanabe Akaike Information Criterion (WAIC):
\begin{align}
    -\text{WAIC}(\vx) = -\bb{E}_{\vtheta}[\log p_\vtheta(\vx)] + \text{Var}_\vtheta[\log p_\vtheta (\vx)]
\end{align}
where the expectations are computed using independently trained model ensembles. This method does not assign higher negative WAIC values for OoD samples in some simple cases (such as learning a mixture of Gaussians~\citep{choi2018waic})
, but is empirically observed it to be effective for certain datasets and generative models.

\paragraph{Likelihood-based permutation tests}



In the context of detecting adversarial examples, \citet{song2017pixeldefend}
considered using permutation tests statistics to determine whether an input $\vx'$ comes from $p(\vx)$ or not. 
One such test statistic 
uses the rank of $p_\vtheta(\vx')$ in $\{p_\vtheta(\vx_1), \cdots, p_\vtheta(\vx_N)\}$, where $\{\vx_i\}_{i=1}^{N}$ is a (training) dataset of $N$ samples from $p(\vx)$:
\begin{align}
    T_{\mathrm{perm}} =  T(\vx'; \vx_1, \ldots, \vx_N) \triangleq \left \vert \sum_{i=1}^{N} \bb{I}[p_\vtheta(\vx_i) \leq p_\vtheta(\vx')] - \frac{N}{2} \right \vert \in [0, N/2 + 1]
\end{align}
where $\bb{I}$ is the indicator function. 
When $T_{\mathrm{perm}}(\vx)$ is large, $\vx$ is assumed to be OoD. This addresses the case where higher likelihood is observed for certain out-of-distribution samples (such as the SVHN vs. CIFAR10 example in~\citet{nalisnick2018do}), as both ``high-likelihood'' and ``low-likelihood'' samples will have $T_{\mathrm{perm}}$ close to $N/2$ (because of the absolute value) and be identified as OoD.

\section{Out-of-distribution samples and model mis-specification}
\label{sec:likelihood}
\citet{nalisnick2018do} find it surprising that deep generative models assign higher likelihood to out-of-distribution samples, given their successful generalizing on a test dataset. However, we argue that this is to be expected when the unerlying generative model $p_\theta(\vx)$ 
is mis-specified, and perhaps this is more likely than previously anticipated in the literature.

We consider a simple example where a mis-specified model would overestimate the likelihood of out-of-distribution samples (Figure~\ref{fig:demo}, right). Let $p(\vx)$ be a uniform mixture of two Gaussians on $\bb{R}^1$, $\gN(-1.0, 0.01)$ and $\gN(1.0, 0.01)$, and $p_\theta(\vx) = \gN(\mu, \sigma^2)$ be our model with parameters $(\mu, \sigma)$. The maximum likelihood solution is $(\mu^\star, \sigma^\star) \approx (0.0, 1.0)$, where mode-covering behavior occurs. $\vx = 0$ has the largest likelihood under $p_\theta(\vx)$, yet it is highly atypical in the original distribution. This is also a failure mode for WAIC as discussed in \citep{choi2018waic}. 


Given the complexity of high-dimensional image distributions, existing likelihood-based generative models are likely to be mis-specified, which invalidates the use of likelihood estimates to perform OoD detection~\citep{white1982maximum}. 
Moreover, model selection is often based on the holdout method~\citep{arlot2010survey}, in which we evaluate log-likelihood over a validation set sampled from $p(\vx)$, 
 but not over out-of-distribution samples. As we do not know the entropy of $p(\vx)$, we can never check whether $\KL(p(\vx) \Vert p_\vtheta(\vx)) \approx 0$ after training. Because of the mode-seeking nature of $\KL$, an alternative distribution $q(\vx) \neq p(\vx)$ can have even lower KL-divergence with $p_\vtheta(\vx)$, i.e., $\KL(q(\vx) \Vert p_\vtheta(\vx)) \ll \KL(p(\vx) \Vert p_\vtheta(\vx))$ even if $p_\vtheta(\vx)$ is trained on samples from $p(\vx)$. In Appendix~\ref{sec:permutation}, we demonstrate that even more sophisticated tests over log-likelihood, such as likelihood-based permutation tests, cannot detect out-of-distribution samples effectively. 

\section{Detecting out-of-distribution samples with batch normalization}
\label{sec:batch-norm-experiments}


In the following sections, we argue that by taking advantage of batch normalization (BatchNorm),   out-of-distribution samples can be detected with existing deep generative models.
For a batch of inputs $\vz = \{\vz_i\}_{i=1}^{b}$ of batch size $b$, batch normalization~\citep{ioffe2015batch} performs normalization over the inputs followed by a parametrized affine transformation:
$$
\text{BatchNorm}(\vz; \gamma, \beta, \epsilon) = \frac{\vz - \bb{E}[\vz]}{\sqrt{\Var[\vz] + \epsilon}} \cdot \gamma + \beta
$$
where $\gamma$, $\beta$ are trainable parameters, and $\epsilon$ is a hyperparameter for numerical stability. The mean $\bb{E}[\vz]$ and variance $\Var[\vz]$ are computed over a single batch in \textit{training mode}, and over the entire training set in \textit{evaluation mode}\footnote{This is typically done by keeping an exponential moving average over the batch statistics during training.}. For deep generative models using batch normalization, existing literature evaluate the log-likelihood on some validation set in \textit{evaluation mode}. The role of \textit{training mode} is generally perceived to accelerate optimization~\citep{kohler2018towards,santurkar2018how}, similar to supervised learning.

\subsection{Training mode decreases likelihood of OoD batches}

\emph{Our central observation is that likelihoods assigned by the model under \textit{training mode} and \textit{evaluation mode} 
are significantly different for OoD samples, but much less so for in-distribution samples}, i.e., when we use BatchNorm 
statistics computed over a single batch.
We evaluate log-likelihood (measured in bits per dimension, BPD\footnote{BPD is defined by the negative log-likelihood divided by the number of dimensions~\citep{theis2015a}; if the likelihood is measured in nats, then an additional division by $\ln 2$ is needed.}) on Fashion MNIST, MNIST, CIFAR and SVHN test sets with several models including RealNVP, VAE and PixelCNN++\footnote{We use an implementation where we replace weight normalization with batch normalization for PixelCNN++, see \href{https://github.com/pclucas14/pixel-cnn-pp}{https://github.com/pclucas14/pixel-cnn-pp}.}. Pre-activation batch normalization layers are used within the residual block for RealNVP and PixelCNN++, and before each convolutional layer for VAE. The datasets are evaluated in both \textit{evaluation mode} and \textit{training mode}.

\begin{table}[t]
\centering
\caption{Log-likelihood (measured in bits per dimension) calculated with RealNVP, VAE, PixelCNN++ models on MNIST, Fashion MNIST, CIFAR10, and SVHN test sets. We train RealNVP and VAE on FashionMNIST, and train RealNVP and PixelCNN on CIFAR10. 
We report likelihood results with \textit{training mode}, \textit{evaluation mode}, and their difference ($\Delta$).}
\begin{subtable}[h]{0.5\textwidth}
\centering
\caption{RealNVP (trained on Fashion MNIST)}
\begin{tabular}{l|cc|c}
\toprule
    Evaluated on       & BN Mode  & BPD & $\Delta$ \\\midrule
    \multirow{2}{*}{Fashion MNIST}  & \textit{evaluation} & 2.92 & \multirow{2}{*}{0.02} \\
     & \textit{training} & 2.94 &      \\\midrule
    \multirow{2}{*}{MNIST}         & \textit{evaluation} & 1.74 & \multirow{2}{*}{7.73} \\
             & \textit{training} & 9.47 &      \\\bottomrule
\end{tabular}
\end{subtable}
\vspace{0.5em}
\begin{subtable}[h]{0.45\textwidth}
\centering
\caption{RealNVP (trained on CIFAR10)}
\begin{tabular}{l|cc|c}
\toprule
    Evaluated on       & BN Mode  & BPD & $\Delta$ \\\midrule
    \multirow{2}{*}{CIFAR10}  & \textit{evaluation} & 3.48 & \multirow{2}{*}{0.03} \\
     & \textit{training} & 3.51 &      \\\midrule
    \multirow{2}{*}{SVHN}         & \textit{evaluation} & 2.44 & \multirow{2}{*}{8.56} \\
             & \textit{training} & 11.10 &      \\\bottomrule
\end{tabular}
\end{subtable}
\begin{subtable}[h]{0.5\textwidth}
\centering
\caption{VAE (trained on Fashion MNIST)}
\begin{tabular}{l|cc|c}
\toprule
    Evaluated on       & BN Mode  & BPD & $\Delta$ \\\midrule
    \multirow{2}{*}{Fashion MNIST}  & \textit{evaluation} & 3.19 & \multirow{2}{*}{0.01} \\
     & \textit{training} & 3.20 &      \\\midrule
    \multirow{2}{*}{MNIST}         & \textit{evaluation} & 1.97 & \multirow{2}{*}{6.53} \\
             & \textit{training} & 8.50 &      \\\bottomrule
\end{tabular}
\end{subtable}
~
\begin{subtable}[h]{0.45\textwidth}
\centering
\caption{PixelCNN++ (trained on CIFAR10)}
\begin{tabular}{l|cc|c}
\toprule
    Evaluated on       & BN Mode  & BPD & $\Delta$ \\\midrule
    \multirow{2}{*}{CIFAR10}  & \textit{evaluation} & 3.21 & \multirow{2}{*}{0.12} \\
     & \textit{training} & 3.33 &      \\\midrule
    \multirow{2}{*}{SVHN}         & \textit{evaluation} & 2.16 & \multirow{2}{*}{1.88} \\
             & \textit{training} & 4.04 &      \\\bottomrule
\end{tabular}
\end{subtable}
\label{tab:batch-norm}
\end{table}

In Table~\ref{tab:batch-norm} we found that in \textit{evaluation mode} MNIST samples have smaller BPD than Fashion MNIST for models trained on Fashion MNIST; SVHN samples have smaller BPD than CIFAR10 for models trained on CIFAR10, which agrees with the observations in~\citep{nalisnick2018do}.
However, as we change from \textit{evaluation mode} to \textit{training mode}, the log-likelihood of the in-distribution samples does not change significantly, while that of OoD samples plummets (see $\Delta$ in Table~\ref{tab:batch-norm}). 
In Appendix~\ref{app:batchnorm-add}, we demonstrate that this sudden decrease in log-likelihood not only happens for samples from other datasets, but also adversarial samples that are able to fool likelihood models and likelihood-based permutation tests. This empirically suggests that simply by using \textit{training mode} during evaluation, we are able to consistently detect a batch of out-of-distribution samples. We further note that no special modifications is made over the training procedure, and no additional OoD datasets are required. 

\begin{figure}[h]
\centering
  \begin{minipage}[c]{0.35\textwidth}
    \includegraphics[width=\textwidth]{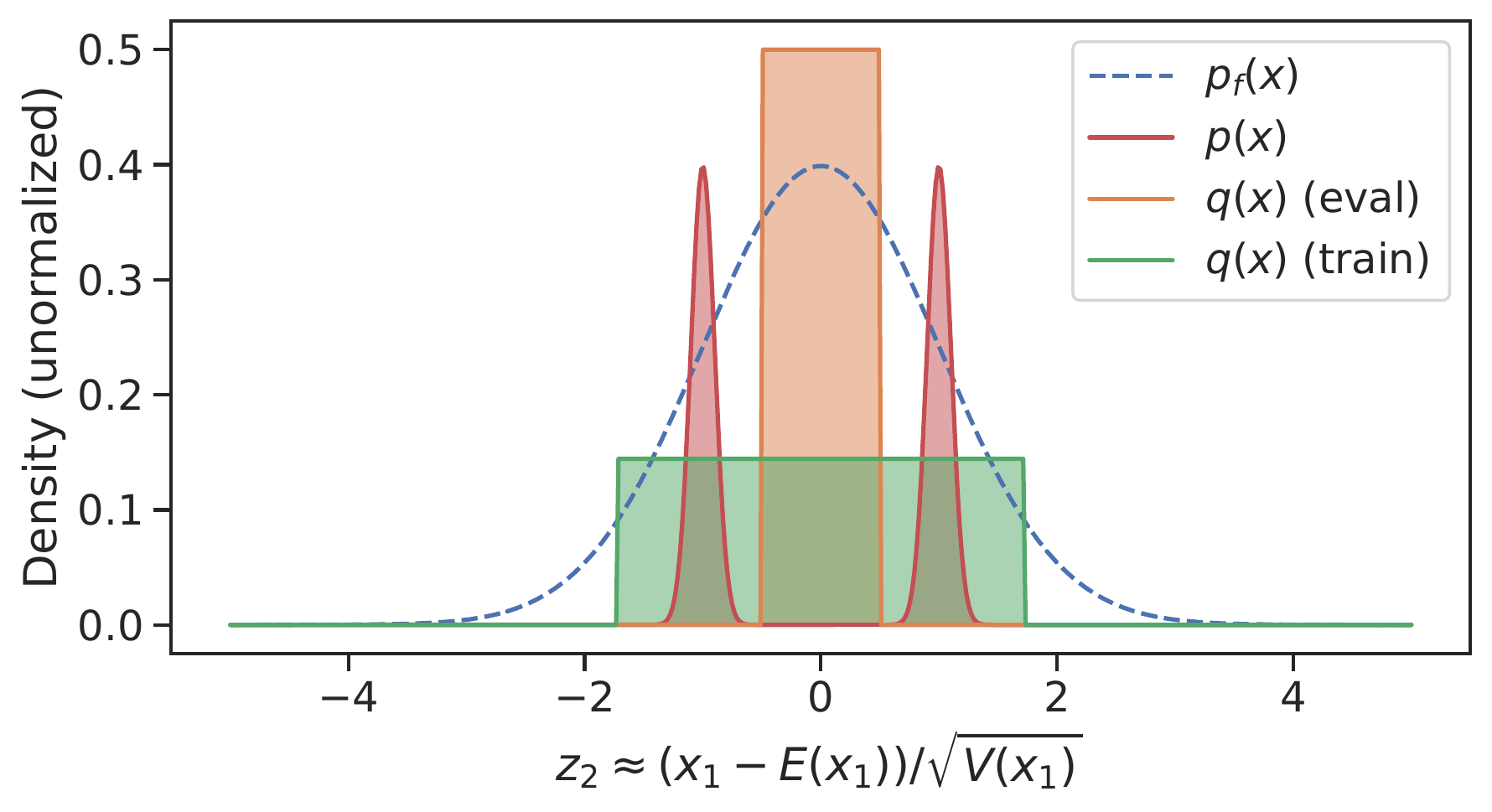}
  \end{minipage}
  ~
  \begin{minipage}[c]{0.55\textwidth}
    \caption{
       Example demonstrating how BatchNorm mitigates high-likelihood in OoD distributions. $p(x)$, $p_f(x)$ and $q(x)$ denote the original, model and OoD distributions. Distribution of $z_2$ are different for $q(x)$ under the two modes of BatchNorm, leading to different likelihood results. 
    } \label{fig:1d_demo_bn}
  \end{minipage}
\end{figure}

\subsection{Explaining the effectiveness of batch normalization}
\label{sec:batchperm-explain}
Why is batch normalization effective at detecting OoD batches while other likelihood-based generative models (e.g., Glow trained with ActNorm) fail? First we provide an example to illustrate why \textit{training mode} and \textit{evaluation mode} provide different likelihood estimates for OoD distributions. Then we argue that with \textit{training mode}, the learned generative model no longer assumes samples are i.i.d., which differs from models without BatchNorm.


\subsubsection{Example of BatchNorm Mitigating Model Mis-specification}
\label{sec:example}

In Figure~\ref{fig:1d_demo_bn}, we illustrate a case where BatchNorm could be useful for OoD detection.
Suppose we try to learn some distribution $p(\vx)$ with a 2-d flow model with one coupling layer~\citep{dinh2016density}: $f: (x_1, x_2) \mapsto (z_1, z_2)$ where $z_1 = x_1$, $z_2 = x_2 + (x_1 - E(x_1)) / \sqrt{V(x_1)} \cdot \gamma + \beta$ with learnable parameters $\gamma, \beta$; $E(x_1)$, $V(x_1)$ are the empirical mean and variance of a batch of $x_1$ in \textit{training mode}, and the mean and variance of $p(x_1)$ in \textit{evaluation mode}. To simplify our analysis, 
We assume that the training distribution $p(x_1, x_2) = p(x_1) p(x_2)$ where $p(x_1)$ is a mixture of two Gaussians ($\gN(-1, \sigma^2)$ and $\gN(1, \sigma^2)$ with $\sigma \to 0$) and $p(x_2) = \delta_0(x_2)$ (Dirac delta) and batch sizes are infinite; hence $z_2 \approx (x_1 - E(x_1)) / \sqrt{V(x_1)} \cdot \gamma + \beta$. Since the Jacobian of $f$ is triangular and has a determinant of 1, the likelihood of the flow model is simply $p_f(x_1, x_2) = p(f(x_1, x_2))$; we assume the prior $p(z_1, z_2)$ is standard Gaussian $\mathcal{N}(0, I)$. The optimal parameters are then $\gamma \approx 1$\footnote{By maximum likelihood, $\gamma \approx \arg \min_s -s^2/2 + \log s = 1$.} and $\beta = 0$.


Now let us consider $q(x_1, x_2) = q(x_1) p(x_2)$ where $q(x_1) = \mathrm{Uniform}(-0.5, 0.5)$. In \textit{evaluation mode}, $z_2 \approx x_1$ and $\bb{E}_q[\log p_f(\vx)] \approx -1.92$; in \textit{training mode}, $z_2 \approx \sqrt{12} x_1$, which decreases the log-likelihood to $\bb{E}_q[\log p_f(\vx)] \approx -2.38$. We include a more detailed analysis in Appendix~\ref{app:example}. 
In this case, BatchNorm manages to learn a global property of the input distribution (variance), and use this to counteract major distributional shifts (e.g. the proposed uniform distribution $q(x)$
). The effect leads to lower log-likelihoods in \textit{training mode} compared to that in \textit{evaluation mode}.  

\subsubsection{Revisiting the i.i.d. Assumptions for Batch-normalized Generative Models}

While it is difficult to analytically characterize the change of likelihood estimates in \textit{training mode} for high-dimensional images and deep neural networks, we provide a probabilistic interpretation as to why \textit{training mode} exhibits different behaviors than \textit{evaluation mode}.

\paragraph{\textit{Training mode} breaks the \emph{i.i.d.} data assumption} In \textit{training mode}, the output for a particular sample $\vx_j$ is affected by other samples in the same batch (denoted as $\vx_{-j}$), so they are not treated as i.i.d for the generative models in \textit{training mode}. 
This suggest that using BatchNorm when training generative models not only changes the optimization landscape but also modifies the objective: it is no longer $\frac{1}{N} \sum_{i=1}^{N} \log p_\vtheta(\vx_i)$ where each sample is assumed to be i.i.d.

\paragraph{Pseudo-likelihood perspective of \textit{training mode}} How should we interpret the training objective with batch normalization?
For each sample $\vx_j$, its corresponding ``likelihood'' objective depends on \textit{the other samples from the same batch}; we denote this as $\ell_\vtheta(\vx_j; \vx_{-j})$. First, we show that the integral of $\exp(\ell_\vtheta(\vx_j; \vx_{-j}))$ over any particular sample ($\vx_{j}$) when we fix the other samples in the batch ($\vx_{-j}$) is one; this allows us to treat $\exp(\ell)$ as a probability density function\footnote{For VAE, we assume that the inference distribution is always optimal, i.e., there is no gap between the evidence lower bound and the log-likelihood, so we can treat the objective as exact maximum likelihood.}.

\begin{restatable}{proposition}{normalization}
\label{prop:normalization}
Let $b \in \bb{N}, b > 1$ be the batch size, and for all possible batch of samples of size $(b-1)$, denoted as $\vx_{-j}$,
if $\ell_\theta(\vx_j; \vx_{-j})$ is the training objective over $\vx_j$ for a likelihood-based generative model
 with batch normalization, then:
$$
\int_{\vx_j} \exp(\ell_\vtheta(\vx_{j}; \vx_{-j})) \mathrm{d} \vx_j = 1.
$$
where we consider likelihood-based generative models that sample either via some parametrized distribution (such as VAE and PixelCNN) or via parametrized invertible transformations constructed via affine coupling layers (such as RealNVP).
\end{restatable}


We defer the proofs in Appendix~\ref{app:normalization}. Moreover, if we interpret $\ell_\theta(\vx_j; \vx_{-j})$ as the \emph{conditional} log-likelihood under a certain joint generative model $\tilde{p}_{\vtheta}$ over an entire batch, our existing ``log-likelihood'' objective is a surrogate to the log-likelihood objective over the joint distribution $\tilde{p}_\vtheta(\vx_1, \vx_2, \ldots, \vx_b)$.

\begin{restatable}{proposition}{pseudo}
\label{prop:pseudo}
There exists a joint distribution $\tilde{p}_\vtheta(\vx_1, \vx_2, \ldots, \vx_b)$ such that for all $j$,  
$\tilde{p}_\vtheta(\vx_j \vert \vx_{-j}) \rightarrow \ell_\vtheta(\vx_j; \vx_{-j})$ as $b \rightarrow \infty$.
Then, the objective for one batch $\{x_j\}_{j=1}^{b}$ in \textit{training mode}
\begin{gather}
    \gL_{\text{train}}(\{x_j\}_{j=1}^{b}; \vtheta) = \sum_{j=1}^{b} \ell_\vtheta(\vx_j \vert \vx_{-j}) \triangleq \sum_{j=1}^{b} \log \tilde{p}_\vtheta(\vx_j \vert \vx_{-j}) 
\end{gather}
is the pseudo-log-likelihood for the joint distribution $\tilde{p}_\vtheta(\vx_1, \vx_2, \ldots, \vx_b)$ as $b \rightarrow \infty$.
\end{restatable}


\paragraph{\textit{Training mode} versus \textit{evaluation mode}} In \textit{evaluation mode}, the samples are treated as i.i.d., since the batch statistics is fixed and samples within the same batch do not affect each other.
The objective for one batch $\{x_j\}_{j=1}^{b}$ evaluated in \textit{evaluation mode} is simply:
\begin{align}
    \gL_{\text{eval}}(\{x_j\}_{j=1}^{b}; \vtheta) = \sum_{j=1}^{b} \log p_\vtheta(\vx_j)
\end{align}

It is evident that $\gL_{\text{train}}$ and $\gL_{\text{eval}}$ are not the same objective due to the differences in i.i.d. assumptions for samples within the same batch. $\gL_{\text{train}}$ depends on the input batch statistics, whereas $\gL_{\text{eval}}$ does not. If we consider OoD batches that exhibit major distributional shifts, then it is likely that $\gL_{\text{train}}$ differs from $\gL_{\text{eval}}$.
From the empirical evidence in Table~\ref{tab:batch-norm}, OoD samples have much lower BPD on \textit{training mode} than on \textit{evaluation mode}, whereas the same does not happen for in-distribution samples. 


\section{A permutation test based on batch normalization}
\label{sec:batchperm}
Assume that during testing we have access to some ``test dataset'' $\hat{\mX}$, and our goal is to identify for each test sample $\hat{\vx} \in \hat{\mX}$ whether it is OoD or not.
Based on the observations over OoD batches, we propose a permutation test 
for each test sample $\hat{\vx}$. 
\paragraph{Interpolation between \textit{training} and \textit{evaluation}} First, we propose a procedure that interpolates between \textit{training mode} and \textit{evaluation mode}, generalizing the settings in Section~\ref{sec:batch-norm-experiments}.
We randomly obtain a batch of samples $\hat{\mX}_{b_1}$ from $\hat{\mX}$ with size $b_1$ and a batch of samples ${\mX}_{b_0}$ from $p(\vx)$ with size $b_0$ (e.g., from the training dataset $\mX$ used to fit the generative model). 
For each $\hat{\vx}_{j} \in \hat{\mX}_{b_1}$ indexed by $j$, we compute its log-likelihood under \textit{training mode} by mixing $\hat{\mX}_{b_1}$ and ${\mX}_{b_0}$ in the same batch:
\begin{align}
    \log \tilde{p}_\vtheta(\hat{\vx}_{j}| \vx_{-j}) \quad \text{where} \quad \vx_{-j} = ({\mX}_{b_0} \cup \hat{\mX}_{b_1}) \setminus \{\hat{\vx}_j\}
    \label{eq:mix-batch}
\end{align}
Let $b = b_1 + b_0$ be the batch size and let $r = b_1 / b$ be the ratio of the test samples in the entire batch. 
We note that the above procedure interpolates evaluation with \textit{training mode} and evaluation with \textit{evaluation mode}. As $r \to 0$, the effect on the batch statistics is dominated by the training samples $\mX_{b_0}$, so Equation~\ref{eq:mix-batch} converges to \textit{evaluation mode}; as $r \to 1$, the effect on the batch statistics is dominated by the test samples $\hat{\mX}_{b_1}$, so Equation~\ref{eq:mix-batch} converges to \textit{training mode}.
As we increase $r$ from 0 to 1, we are moving from \textit{evaluation mode} to \textit{training mode}. 

The observations in Section~\ref{sec:batch-norm-experiments} suggest that $\log \tilde{p}_\vtheta(\hat{\vx}_{j}| \vx_{-j})$ should be relatively stable if $q(\vx)$ is close to $p(\vx)$ and increase significantly if $q(\vx)$ is far from $p(\vx)$, which is indeed the case empirically. We consider evaluating the average log-likelihood on the test samples when the ratio $r$ is varied, and show the results on several datasets and models in Figure~\ref{fig:batch-perm-verify}. As we increase the ratio of test samples, the corresponding log-likelihood decreases dramatically for OoD samples, but is relatively stable for in-distribution samples.

\begin{figure}
\centering
\begin{subfigure}{0.45\textwidth}
    \centering
    \includegraphics[width=\textwidth]{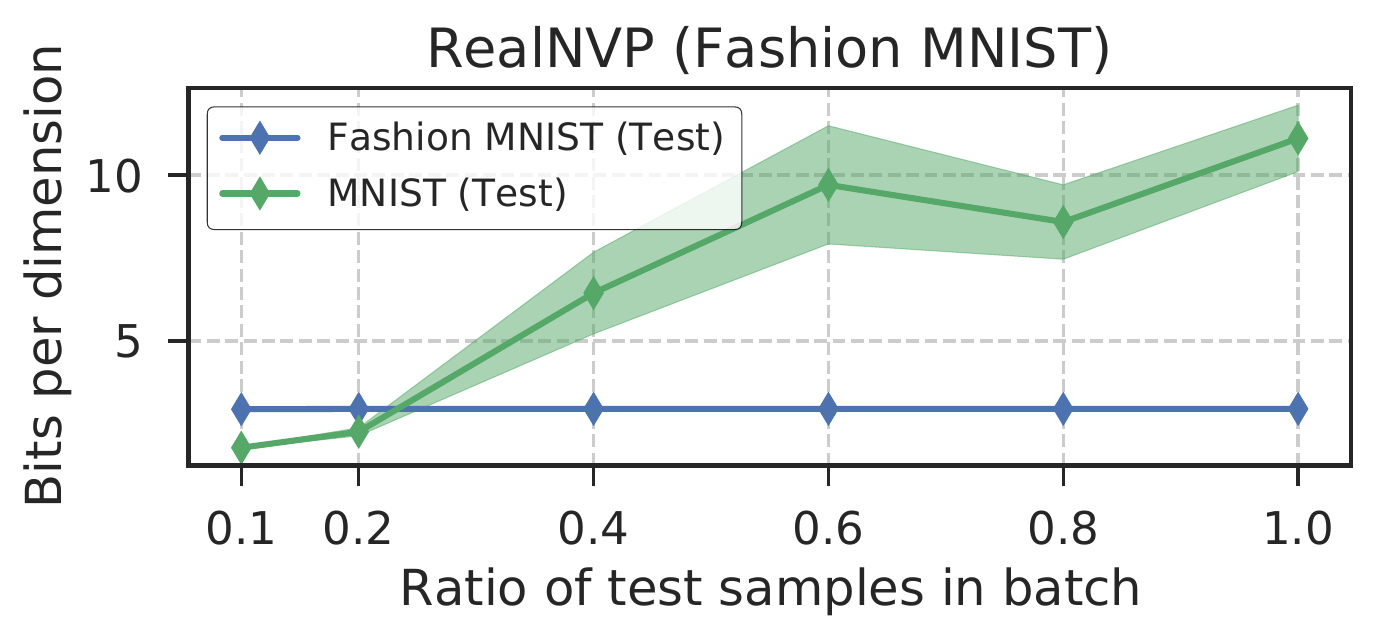}
\end{subfigure}
~
\begin{subfigure}{0.45\textwidth}
    \centering
    \includegraphics[width=\textwidth]{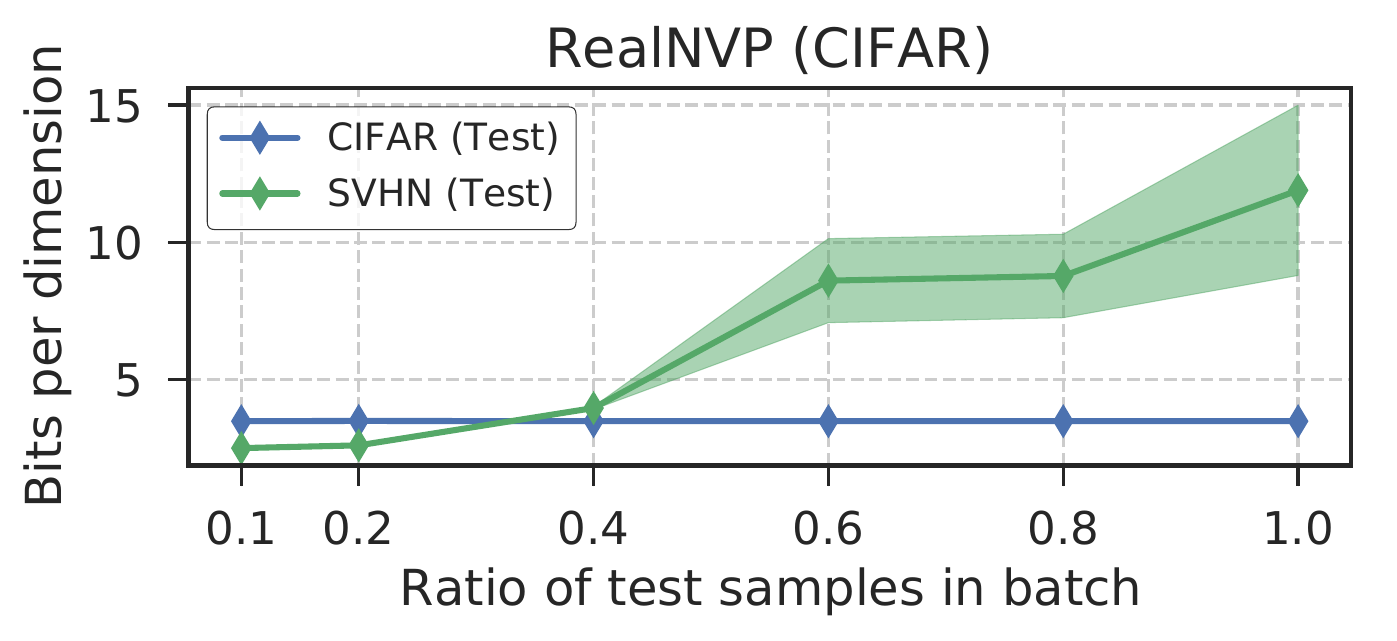}
\end{subfigure}
~
\begin{subfigure}{0.45\textwidth}
    \centering
    \includegraphics[width=\textwidth]{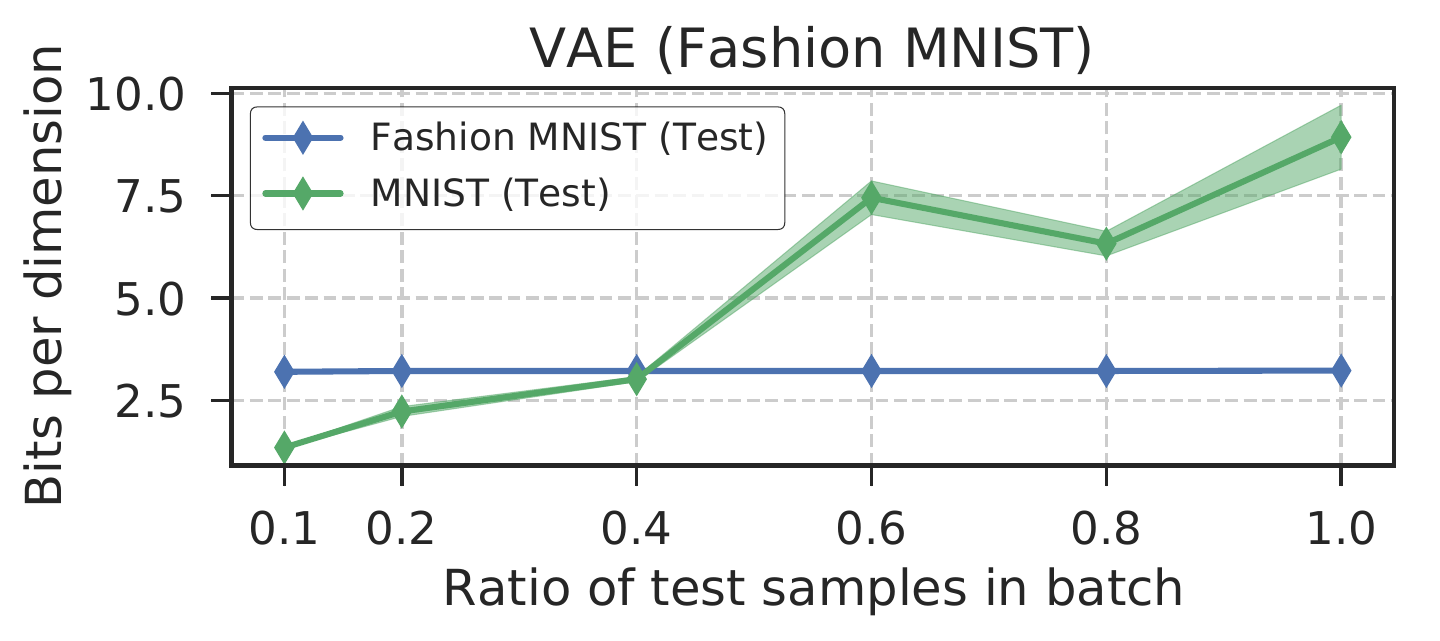}
\end{subfigure}
~
\begin{subfigure}{0.45\textwidth}
    \centering
    \includegraphics[width=\textwidth]{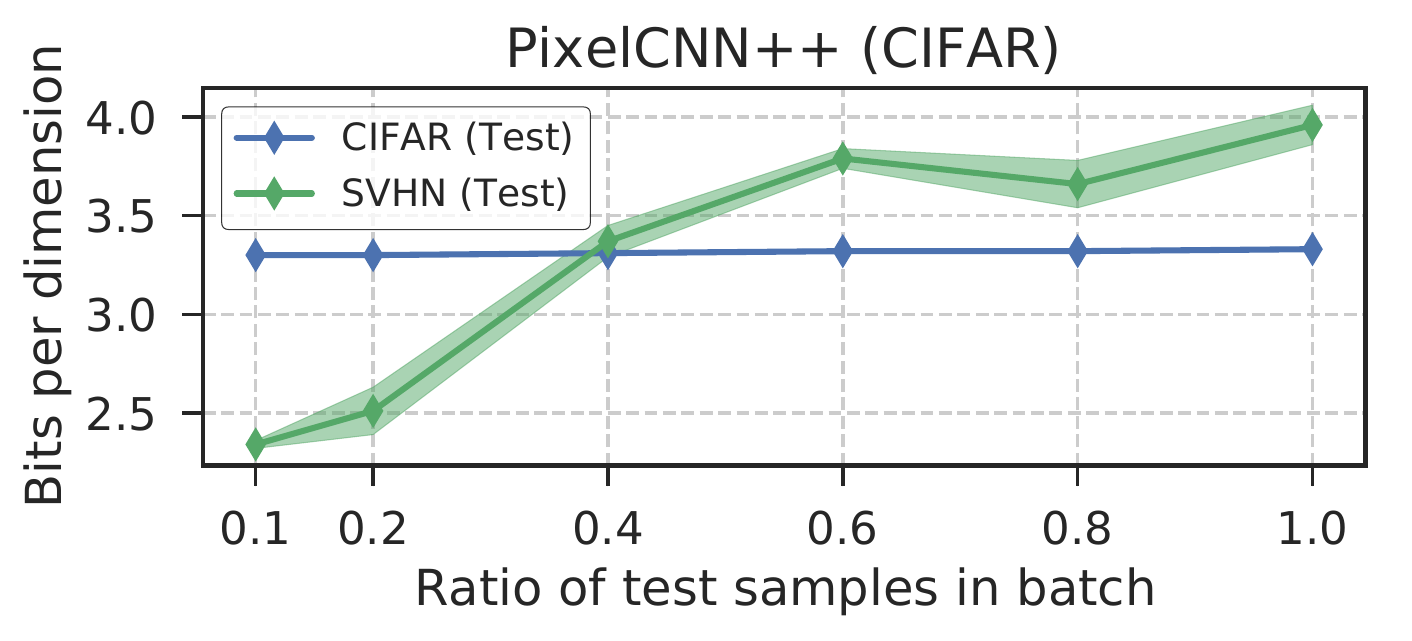}
\end{subfigure}
\caption{Average BPD ($-\log \tilde{p}_\vtheta(\vx_j | \vx_{-j})$) for test samples with varying ratios of test samples in the batch. The BPD of in-distribution samples do not increase as the ratio increase, yet that of OoD samples increase significantly. This justifies the use of $\Delta_{b, r_1, r_2}({\vx}')$ for OoD detection.}
\label{fig:batch-perm-verify}
\end{figure}

\paragraph{Permutation test with BatchNorm} We proceed to propose a permutation test statistics based on \textit{training mode} evaluation. 
For a test sample $\hat{x}$, a fixed batch size $b$ and fixed ratio $r$, we select $(rb - 1)$ samples from $q(\vx)$ and $(1 - r) b$ samples from $p(\vx)$. The expected conditional likelihood of $\hat{\vx}$ is used as the test statistic:
\begin{align}
S_{b, r}(\vx') = \bb{E}_{p, q}\left[\log \tilde{p}_\vtheta(\hat{\vx}|\mX_{(1-r)b}, \hat{\mX}_{rb-1})\right] \quad \text{where} \quad \mX_{(1-r)b} \sim p(\vx), \hat{\mX}_{rb-1} \sim q(\vx)
\end{align}

In practice, we use monte carlo estimates for the expectation in $S_{b, r}(\vx')$. We select two different ratios $r_1, r_2 \in (0, 1)$ where $r_1 < r_2$, and compute the differences in $S_{b, r}(\vx')$:
\begin{align}
    \Delta_{b, r_1, r_2}({\vx}') = \left\vert S_{b, r1}(\vx') - S_{b, r2}(\vx') \right\vert
\end{align}
We proceed to use the rank of $\Delta_{b, r_1, r_2}(\hat{\vx})$ in the training set as our test statistic:
\begin{align}
    T_{b, r_1, r_2} = T({\vx}'; \vx_1, \ldots, \vx_N) \triangleq \sum_{i=1}^{N} \bb{I}[\Delta_{b, r_1, r_2}(\vx_i) \leq \Delta_{b, r_1, r_2}(\vx')]
\end{align}

For in-distribution samples, we expect $\Delta_{b, r_1, r_2}$ to be small across all choices of $r_1$ and $r_2$; for out-of-distribution samples, $\Delta_{b, r_1, r_2}$ should be large if $(r_2 - r_1)$ is large. Therefore, if $q(\vx) = p(\vx)$, then the statistics computed over the $\vx \sim q(\vx)$ should be approximately uniformly distributed over $[0, N]$, whereas if $q(\vx) \neq p(\vx)$, then the statistics computed over out-of-distribution samples would be concentrated around $N$.  

\section{Experiments}
\label{sec:experiments}

\begin{table}[h]
    \centering
    \caption{Out-of-distribution classification evaluated with AUC (left) and Average Precision (right). Higher is better. Rotation denotes $q(\vx)$ uses images in $p(\vx)$ yet each image is rotated randomly by $d \in (90, 270)$ degrees.}
    
    \begin{tabular}{ccc|cccc}
    \toprule
    $p(\vx)$ & Model & $q(\vx)$ & $\log p_\vtheta(\vx)$ & $T_{\text{perm}}$ & WAIC & Ours \\\midrule
    \multirow{4}{*}{\makecell{Fashion \\ MNIST}} & \multirow{4}{*}{RealNVP} 
        & Rotation & 0.76 / 0.78 & 0.64 / 0.68 &  \textbf{0.99} / \textbf{0.99} & \textbf{0.99} / \textbf{0.99} \\
     &  & MNIST    & 0.10 / 0.32 & 0.78 / 0.71 & 0.24 / 0.38  & \textbf{1.00} / \textbf{1.00} \\
     &  & Omniglot & 0.05 / 0.06 & 0.86 / 0.80 & 0.97 / 0.95 & \textbf{1.00} / \textbf{1.00} \\
     &  & KMNIST   & 0.47 / 0.46 & 0.45 / 0.45 & 0.63 / 0.65 & \textbf{1.00} / \textbf{1.00} \\\midrule
    \multirow{4}{*}{\makecell{Fashion \\ MNIST}} & \multirow{4}{*}{VAE}
        & Rotation & 0.73 / 0.72 & 0.61 / 0.64 & 0.94 / 0.95 & \textbf{0.97} / \textbf{0.98} \\
     &  & MNIST    & 0.13 / 0.33 & 0.73 / 0.68 & 0.56 / 0.64 & \textbf{1.00} / \textbf{1.00} \\
     &  & Omniglot & 0.00 / 0.06 & 0.99 / 0.96 & 0.90 / 0.83 & \textbf{1.00} / \textbf{1.00} \\
     &  & KMNIST   & 0.55 / 0.54 & 0.50 / 0.50 & 0.84 / 0.87 & \textbf{1.00} / \textbf{1.00} \\\midrule
     \multirow{4}{*}{CIFAR} & \multirow{4}{*}{RealNVP} 
        & Rotation & 0.87 / 0.87 & 0.79 / 0.79 & 0.99 / 0.98 & \textbf{1.00} / \textbf{1.00} \\
     &  & SVHN     & 0.07 / 0.52 & 0.86 / 0.82 & 0.16 / 0.55 & \textbf{1.00} / \textbf{1.00}  \\
     &  & ImageNet & 0.51 / 0.52 & 0.50 / 0.51 & 0.58 / 0.59 & \textbf{0.98} / \textbf{0.97} \\
     &  & LSUN     & 0.70 / 0.39 & 0.58 / 0.56 & 0.60 / 0.28 & \textbf{0.99} / \textbf{0.98} \\\midrule
     \multirow{4}{*}{CIFAR} & \multirow{4}{*}{PixelCNN++}
        & Rotation & 0.77 / 0.75 & 0.67 / 0.63 & 0.90 / 0.85 & \textbf{0.99} / \textbf{0.99} \\
     &  & SVHN     & 0.10 / 0.32 & 0.86 / 0.77 & 0.09 / 0.53 & \textbf{0.99} / \textbf{0.99} \\
     &  & ImageNet & 0.51 / 0.51 & 0.49 / 0.50 & 0.66 / 0.69 & \textbf{0.89} / \textbf{0.87} \\
     &  & LSUN     & 0.72 / 0.69 & 0.60 / 0.58 & 0.78 / 0.74 & \textbf{0.98} / \textbf{0.97} \\\bottomrule
    \end{tabular}
    \label{tab:batch-perm}
\end{table}

We verify the effectiveness of our proposed test statistic on several datasets and models, including RealNVP, VAE and PixelCNN++, and compare against several baselines including log-likelihood, permutation test and WAIC. First, we train a model $\tilde{p}_\vtheta(\vx)$ on a dataset $p(\vx)$; then we proceed to obtain the test statistic for each sample in $q(\vx)$, where $q(\vx)$ are several different datasets. For each $q(\vx) \neq p(\vx)$, we consider a binary classification problem, where the prediction is the test statistic for each sample $\vx$, and the label is whether $\vx$ is out-of-distribution (label 1) or not (label 0); larger prediction values indicate $\vx$ is more likely to be out-of-distribution. We evaluate the area under the ROC curve (AUC) and average precision (AP) for each binary classification task, following the procedure in~\citep{choi2018waic}. We select $r_1 = 0.1$ and $r_2 = 0.9$ for our proposed test statistic. 

The OoD detection results are shown in Table~\ref{tab:batch-perm}. As expected, using $\log p_\vtheta(\vx)$ results in poor AUC and AP in cases where the OoD samples have higher log-likelihoods; $T_{\text{perm}}$ outperforms $\log p_\vtheta(\vx)$ in cases where we also consider high log-likelihood samples as OoD; WAIC achieves higher AUC / AP than $\log p_\vtheta(\vx)$, yet its improvement is inconsistent across different models (RealNVP improvements are lower than VAE on Fashion MNIST). Our proposed statistic is able to detect all the out-of-distribution samples by achieving near-optimal AUC / AP in most cases; this is most notable on CIFAR vs. ImageNet and Fashion MNIST vs. KMNIST as the sample have very similar likelihood distributions in \textit{evaluation mode}.
In Appendix~\ref{app:r1r2}, we show that our method works well even with $r_1 \to 0$ and $r_2 = 0.15$; therefore, the method is not very sensitive to the $r_1, r_2$ hyperparameters. This suggest that in practice, even when the test samples contains only a small portion of the OoD samples we can still detect them reliably by selecting $r_1 \to 0$ (evaluation mode) and $r_2 \to 1$ (training mode).
related work still needs some work, but there seems to be some bug on overleaf right now\section{Related work}
\label{sec:related}
\paragraph{Task-dependent OoD detection}
Out-of-distribution detection is crucial to applications such as anomaly detection~\citep{pidhorskyi2018generative,hendrycks2016a,vyas2018out}, adversarial defense~\citep{song2017pixeldefend,song2018constructing}, and novelty detection for exploration~\citep{marsland2003novelty,bellemare2016unifying,fu2017ex2}.
In the context of supervised learning, OoD detection methods~\citep{hendrycks2016a,liang2017enhancing,devries2018learning,gal2016dropout,lakshminarayanan2017simple} are applied to prevent poorly-calibrated neural networks~\citep{guo2017calibration} from making high-confidence predictions on nonsensical inputs~\citep{szegedy2013intriguing,goodfellow2014explaining}; these methods are task dependent and are not suitable for task-independent cases, such as exploration over novel states~\citep{fu2017ex2,ostrovski2017count}.

\paragraph{Generative models for task-independent OoD detection} 
For high-dimensional inputs such as images, deep generative models are widely applied for density estimation~\citep{oord2016pixel,dinh2016density,kingma2018glow}, and thus are naturally considered for task-independent OoD detection~\citep{chalapathy2018group,xu2018unsupervised,kliger2018novelty,ostrovski2017count,li2018anomaly}. However, recent empirical evidence suggest that likelihood estimates by popular deep generative models are not reliable enough for OoD detection, even against samples that are not adversarial by construction~\citep{nalisnick2018do,hendrycks2018deep}. 
To address these issues, \citeauthor{vskvara2018generative} propose tuning the hyperparameters of VAEs with additional OoD data.
Choi et al.~\citep{choi2018waic} address this empirically by assuming that OoD samples have higher variance likelihood estimates under different independently trained models. \citeauthor{ren2019likelihood} consider learning a separate density model for potentially confounding background statistics.

\paragraph{Statistical tests for OoD detection} OoD detection can be posed as a hypothesis test where the null hypothesis assumes the data is OoD. One could also consider a more specific goodness-of-fit (GoF) test which determines whether a dataset is drawn from a specific distribution. \citet{nalisnick2019detecting} considers Kernelized Stein discrepancy~\citep{chwialkowski2016kernel,liu2016kernelized} and Maximum-mean discrepancy~\citep{gretton2012kernel} as two GoF tests that scales to generative models. They propose to compare the empirical cross-entropy of the samples with entropy of the dataset, which is related to the permutation test proposed in~\cite{song2017pixeldefend}. Our permutation test can also be cast as a GoF test with $r \to 1$, where all test samples in the same batch with \textit{training mode}, but we can also provide individual OoD predictions via different $r$ values.

\section{Conclusion}
\label{sec:conclusion}
In this work, we revisit the i.i.d. assumptions in generative models trained with batch normalization, which results in an alternative interpretation to the training objective. The difference between the training objective and the evaluation objective explains the intriguing observation that likelihood estimates over OoD examples in \textit{training mode} are much lower than that in \textit{evaluation mode}.
This allows us to develop a permutation test based on batch normalization, with which we can reliably detect OoD examples even for difficult cases such as Fashion MNIST vs. KMNIST.

We argue that batch normalization is merely one approach to introducing non-independence between observation variables in deep generative models. Neural processes~\citep{garnelo2018neural}, for example, considers generative models conditioned on certain observation contexts, which may differ significantly between regular and OoD batches. It would be interesting to investigate if such models could be utilized for more reliable out-of-distribution detection methods for high-dimensional data.

\bibliography{refs}

\begin{thebibliography}{54}
\providecommand{\natexlab}[1]{#1}
\providecommand{\url}[1]{\texttt{#1}}
\expandafter\ifx\csname urlstyle\endcsname\relax
  \providecommand{\doi}[1]{doi: #1}\else
  \providecommand{\doi}{doi: \begingroup \urlstyle{rm}\Url}\fi

\bibitem[Amodei \& Clark(2016)Amodei and Clark]{amodei2016faulty}
Dario Amodei and Jack Clark.
\newblock Faulty reward functions in the wild, 2016.

\bibitem[Arlot et~al.(2010)Arlot, Celisse, et~al.]{arlot2010survey}
Sylvain Arlot, Alain Celisse, et~al.
\newblock A survey of cross-validation procedures for model selection.
\newblock \emph{Statistics surveys}, 4:\penalty0 40--79, 2010.

\bibitem[Arnold \& {Press, S. James}(1989)Arnold and {Press, S.
  James}]{arnold1989compatible}
Barry~C Arnold and {Press, S. James}.
\newblock Compatible conditional distributions.
\newblock \emph{Journal of the American Statistical Association}, 84\penalty0
  (405):\penalty0 152--156, March 1989.
\newblock ISSN 0162-1459.
\newblock \doi{10.1080/01621459.1989.10478750}.

\bibitem[Ball{\'e} et~al.(2016)Ball{\'e}, Laparra, and
  Simoncelli]{balle2016end}
Johannes Ball{\'e}, Valero Laparra, and Eero~P Simoncelli.
\newblock End-to-end optimized image compression.
\newblock \emph{arXiv preprint arXiv:1611.01704}, 2016.

\bibitem[Bellemare et~al.(2016)Bellemare, Srinivasan, Ostrovski, Schaul,
  Saxton, and Munos]{bellemare2016unifying}
Marc Bellemare, Sriram Srinivasan, Georg Ostrovski, Tom Schaul, David Saxton,
  and Remi Munos.
\newblock Unifying {Count-Based} exploration and intrinsic motivation.
\newblock In D~D Lee, M~Sugiyama, U~V Luxburg, I~Guyon, and R~Garnett (eds.),
  \emph{Advances in Neural Information Processing Systems 29}, pp.\
  1471--1479. Curran Associates, Inc., 2016.

\bibitem[Besag(1975)]{besag1975statistical}
Julian Besag.
\newblock Statistical analysis of non-lattice data.
\newblock \emph{Journal of the Royal Statistical Society: Series D (The
  Statistician)}, 24\penalty0 (3):\penalty0 179--195, 1975.

\bibitem[Chalapathy et~al.(2018)Chalapathy, Toth, and
  Chawla]{chalapathy2018group}
Raghavendra Chalapathy, Edward Toth, and Sanjay Chawla.
\newblock Group anomaly detection using deep generative models.
\newblock \emph{arXiv preprint arXiv:1804.04876}, April 2018.

\bibitem[Choi et~al.(2018)Choi, Jang, and Alemi]{choi2018waic}
Hyunsun Choi, Eric Jang, and Alexander~A Alemi.
\newblock {WAIC}, but why? generative ensembles for robust anomaly detection.
\newblock \emph{arXiv preprint arXiv:1810.01392}, October 2018.

\bibitem[Chwialkowski et~al.(2016)Chwialkowski, Strathmann, and
  Gretton]{chwialkowski2016kernel}
Kacper Chwialkowski, Heiko Strathmann, and Arthur Gretton.
\newblock A kernel test of goodness of fit.
\newblock JMLR: Workshop and Conference Proceedings, 2016.

\bibitem[Cortes et~al.(2017)Cortes, DeSalvo, Gentile, Mohri, and
  Yang]{cortes2017online}
Corinna Cortes, Giulia DeSalvo, Claudio Gentile, Mehryar Mohri, and Scott Yang.
\newblock Online learning with abstention.
\newblock \emph{arXiv preprint arXiv:1703.03478}, 2017.

\bibitem[DeVries \& Taylor(2018)DeVries and Taylor]{devries2018learning}
Terrance DeVries and Graham~W Taylor.
\newblock Learning confidence for {Out-of-Distribution} detection in neural
  networks.
\newblock \emph{arXiv preprint arXiv:1802.04865}, February 2018.

\bibitem[Dinh et~al.(2014)Dinh, Krueger, and Bengio]{dinh2014nice}
L~Dinh, D~Krueger, and Y~Bengio.
\newblock {NICE}: Non-linear independent components estimation.
\newblock \emph{arXiv preprint arXiv:1410.8516}, 2014.

\bibitem[Dinh et~al.(2016)Dinh, Sohl-Dickstein, and Bengio]{dinh2016density}
Laurent Dinh, Jascha Sohl-Dickstein, and Samy Bengio.
\newblock Density estimation using real {NVP}.
\newblock \emph{arXiv preprint arXiv:1605.08803}, May 2016.

\bibitem[Fu et~al.(2017)Fu, Co-Reyes, and Levine]{fu2017ex2}
Justin Fu, John Co-Reyes, and Sergey Levine.
\newblock Ex2: Exploration with exemplar models for deep reinforcement
  learning.
\newblock In \emph{Advances in Neural Information Processing Systems}, pp.\
  2577--2587, 2017.

\bibitem[Gal \& Ghahramani(2016)Gal and Ghahramani]{gal2016dropout}
Yarin Gal and Zoubin Ghahramani.
\newblock Dropout as a bayesian approximation: Representing model uncertainty
  in deep learning.
\newblock In \emph{international conference on machine learning}, pp.\
  1050--1059, 2016.

\bibitem[Garnelo et~al.(2018)Garnelo, Schwarz, Rosenbaum, Viola, Rezende,
  Ali~Eslami, and Teh]{garnelo2018neural}
Marta Garnelo, Jonathan Schwarz, Dan Rosenbaum, Fabio Viola, Danilo~J Rezende,
  S~M Ali~Eslami, and Yee~Whye Teh.
\newblock Neural processes.
\newblock \emph{arXiv preprint arXiv:1807.01622}, July 2018.

\bibitem[Goodfellow et~al.(2014)Goodfellow, Shlens, and
  Szegedy]{goodfellow2014explaining}
Ian~J Goodfellow, Jonathon Shlens, and Christian Szegedy.
\newblock Explaining and harnessing adversarial examples.
\newblock \emph{arXiv preprint arXiv:1412.6572}, 2014.

\bibitem[Gretton et~al.(2012)Gretton, Borgwardt, Rasch, Sch{\"o}lkopf, and
  Smola]{gretton2012kernel}
Arthur Gretton, Karsten~M Borgwardt, Malte~J Rasch, Bernhard Sch{\"o}lkopf, and
  Alexander Smola.
\newblock A kernel two-sample test.
\newblock \emph{Journal of Machine Learning Research}, 13\penalty0
  (Mar):\penalty0 723--773, 2012.

\bibitem[Guo et~al.(2017)Guo, Pleiss, Sun, and Weinberger]{guo2017calibration}
Chuan Guo, Geoff Pleiss, Yu~Sun, and Kilian~Q Weinberger.
\newblock On calibration of modern neural networks.
\newblock In \emph{Proceedings of the 34th International Conference on Machine
  Learning-Volume 70}, pp.\  1321--1330. JMLR. org, 2017.

\bibitem[Hendrycks \& Gimpel(2016)Hendrycks and Gimpel]{hendrycks2016a}
Dan Hendrycks and Kevin Gimpel.
\newblock A baseline for detecting misclassified and {Out-of-Distribution}
  examples in neural networks.
\newblock \emph{arXiv preprint arXiv:1610.02136}, October 2016.

\bibitem[Hendrycks et~al.(2018)Hendrycks, Mazeika, and
  Dietterich]{hendrycks2018deep}
Dan Hendrycks, Mantas Mazeika, and Thomas Dietterich.
\newblock Deep anomaly detection with outlier exposure.
\newblock \emph{arXiv preprint arXiv:1812.04606}, December 2018.

\bibitem[Ioffe \& Szegedy(2015)Ioffe and Szegedy]{ioffe2015batch}
Sergey Ioffe and Christian Szegedy.
\newblock Batch normalization: Accelerating deep network training by reducing
  internal covariate shift.
\newblock \emph{arXiv preprint arXiv:1502.03167}, 2015.

\bibitem[Kingma \& Dhariwal(2018)Kingma and Dhariwal]{kingma2018glow}
Diederik~P Kingma and Prafulla Dhariwal.
\newblock Glow: Generative flow with invertible 1x1 convolutions.
\newblock \emph{arXiv preprint arXiv:1807.03039}, July 2018.

\bibitem[Kingma \& Welling(2013)Kingma and Welling]{kingma2013auto}
Diederik~P Kingma and Max Welling.
\newblock {Auto-Encoding} variational bayes.
\newblock \emph{arXiv preprint arXiv:1312.6114v10}, December 2013.

\bibitem[Kliger \& Fleishman(2018)Kliger and Fleishman]{kliger2018novelty}
Mark Kliger and Shachar Fleishman.
\newblock Novelty detection with gan.
\newblock \emph{arXiv preprint arXiv:1802.10560}, 2018.

\bibitem[Kohler et~al.(2018)Kohler, Daneshmand, Lucchi, Zhou, Neymeyr, and
  Hofmann]{kohler2018towards}
Jonas Kohler, Hadi Daneshmand, Aurelien Lucchi, Ming Zhou, Klaus Neymeyr, and
  Thomas Hofmann.
\newblock Towards a theoretical understanding of batch normalization.
\newblock \emph{arXiv preprint arXiv:1805.10694}, 2018.

\bibitem[Krizhevsky \& Hinton(2009)Krizhevsky and
  Hinton]{krizhevsky2009learning}
Alex Krizhevsky and Geoffrey Hinton.
\newblock Learning multiple layers of features from tiny images.
\newblock Technical report, Citeseer, 2009.

\bibitem[Kuleshov et~al.(2018)Kuleshov, Fenner, and
  Ermon]{kuleshov2018accurate}
Volodymyr Kuleshov, Nathan Fenner, and Stefano Ermon.
\newblock Accurate uncertainties for deep learning using calibrated regression.
\newblock In \emph{International Conference on Machine Learning}, 2018.

\bibitem[Lakshminarayanan et~al.(2017)Lakshminarayanan, Pritzel, and
  Blundell]{lakshminarayanan2017simple}
Balaji Lakshminarayanan, Alexander Pritzel, and Charles Blundell.
\newblock Simple and scalable predictive uncertainty estimation using deep
  ensembles.
\newblock In \emph{Advances in Neural Information Processing Systems}, pp.\
  6402--6413, 2017.

\bibitem[Li et~al.(2018)Li, Chen, Goh, and Ng]{li2018anomaly}
Dan Li, Dacheng Chen, Jonathan Goh, and See-kiong Ng.
\newblock Anomaly detection with generative adversarial networks for
  multivariate time series.
\newblock \emph{arXiv preprint arXiv:1809.04758}, 2018.

\bibitem[Liang et~al.(2017)Liang, Li, and Srikant]{liang2017enhancing}
Shiyu Liang, Yixuan Li, and R~Srikant.
\newblock Enhancing the reliability of out-of-distribution image detection in
  neural networks.
\newblock \emph{arXiv preprint arXiv:1706.02690}, June 2017.

\bibitem[Liu et~al.(2016)Liu, Lee, and Jordan]{liu2016kernelized}
Qiang Liu, Jason Lee, and Michael Jordan.
\newblock A kernelized stein discrepancy for goodness-of-fit tests.
\newblock In \emph{International conference on machine learning}, pp.\
  276--284, 2016.

\bibitem[Marsland(2003)]{marsland2003novelty}
Stephen Marsland.
\newblock Novelty detection in learning systems.
\newblock \emph{Neural computing surveys}, 3\penalty0 (2):\penalty0 157--195,
  2003.

\bibitem[Moosavi-Dezfooli et~al.(2017)Moosavi-Dezfooli, Fawzi, Fawzi, and
  Frossard]{moosavi2017universal}
Seyed-Mohsen Moosavi-Dezfooli, Alhussein Fawzi, Omar Fawzi, and Pascal
  Frossard.
\newblock Universal adversarial perturbations.
\newblock In \emph{Proceedings of the IEEE conference on computer vision and
  pattern recognition}, pp.\  1765--1773, 2017.

\bibitem[Nalisnick et~al.(2018)Nalisnick, Matsukawa, Teh, Gorur, and
  Lakshminarayanan]{nalisnick2018do}
Eric Nalisnick, Akihiro Matsukawa, Yee~Whye Teh, Dilan Gorur, and Balaji
  Lakshminarayanan.
\newblock Do deep generative models know what they don't know?
\newblock \emph{arXiv preprint arXiv:1810.09136}, October 2018.

\bibitem[Nalisnick et~al.(2019)Nalisnick, Matsukawa, Teh, and
  Lakshminarayanan]{nalisnick2019detecting}
Eric Nalisnick, Akihiro Matsukawa, Yee~Whye Teh, and Balaji Lakshminarayanan.
\newblock Detecting out-of-distribution inputs to deep generative models using
  a test for typicality.
\newblock \emph{arXiv preprint arXiv:1906.02994}, 2019.

\bibitem[Netzer et~al.(2011)Netzer, Wang, Coates, Bissacco, Wu, and
  Ng]{netzer2011reading}
Yuval Netzer, Tao Wang, Adam Coates, Alessandro Bissacco, Bo~Wu, and Andrew~Y
  Ng.
\newblock Reading digits in natural images with unsupervised feature learning.
\newblock 2011.

\bibitem[Nguyen et~al.(2015)Nguyen, Yosinski, and Clune]{nguyen2015deep}
Anh Nguyen, Jason Yosinski, and Jeff Clune.
\newblock Deep neural networks are easily fooled: High confidence predictions
  for unrecognizable images.
\newblock In \emph{Proceedings of the IEEE conference on computer vision and
  pattern recognition}, pp.\  427--436, 2015.

\bibitem[Ostrovski et~al.(2017)Ostrovski, Bellemare, van~den Oord, and
  Munos]{ostrovski2017count}
Georg Ostrovski, Marc~G Bellemare, A{\"a}ron van~den Oord, and R{\'e}mi Munos.
\newblock Count-based exploration with neural density models.
\newblock In \emph{Proceedings of the 34th International Conference on Machine
  Learning-Volume 70}, pp.\  2721--2730. JMLR. org, 2017.

\bibitem[Pidhorskyi et~al.(2018)Pidhorskyi, Almohsen, and
  Doretto]{pidhorskyi2018generative}
Stanislav Pidhorskyi, Ranya Almohsen, and Gianfranco Doretto.
\newblock Generative probabilistic novelty detection with adversarial
  autoencoders.
\newblock In S~Bengio, H~Wallach, H~Larochelle, K~Grauman, N~Cesa-Bianchi, and
  R~Garnett (eds.), \emph{Advances in Neural Information Processing Systems
  31}, pp.\  6822--6833. Curran Associates, Inc., 2018.

\bibitem[Ren et~al.(2019)Ren, Liu, Fertig, Snoek, Poplin, DePristo, Dillon, and
  Lakshminarayanan]{ren2019likelihood}
Jie Ren, Peter~J Liu, Emily Fertig, Jasper Snoek, Ryan Poplin, Mark~A DePristo,
  Joshua~V Dillon, and Balaji Lakshminarayanan.
\newblock Likelihood ratios for out-of-distribution detection.
\newblock \emph{arXiv preprint arXiv:1906.02845}, 2019.

\bibitem[Salimans et~al.(2017)Salimans, Karpathy, Chen, and
  Kingma]{salimans2017pixelcnn}
Tim Salimans, Andrej Karpathy, Xi~Chen, and Diederik~P Kingma.
\newblock {{PixelCNN++}}: Improving the {PixelCNN} with discretized logistic
  mixture likelihood and other modifications.
\newblock \emph{arXiv preprint arXiv:1701.05517}, January 2017.

\bibitem[Santurkar et~al.(2018)Santurkar, Tsipras, Ilyas, and
  Madry]{santurkar2018how}
Shibani Santurkar, Dimitris Tsipras, Andrew Ilyas, and Aleksander Madry.
\newblock How does batch normalization help optimization?
\newblock \emph{arXiv preprint arXiv:1805.11604}, May 2018.

\bibitem[{\v{S}}kv{\'a}ra et~al.(2018){\v{S}}kv{\'a}ra, Pevn{\`y}, and
  {\v{S}}m{\'\i}dl]{vskvara2018generative}
V{\'\i}t {\v{S}}kv{\'a}ra, Tom{\'a}{\v{s}} Pevn{\`y}, and V{\'a}clav
  {\v{S}}m{\'\i}dl.
\newblock Are generative deep models for novelty detection truly better?
\newblock \emph{arXiv preprint arXiv:1807.05027}, 2018.

\bibitem[Song et~al.(2017)Song, Kim, Nowozin, Ermon, and
  Kushman]{song2017pixeldefend}
Yang Song, Taesup Kim, Sebastian Nowozin, Stefano Ermon, and Nate Kushman.
\newblock {PixelDefend}: Leveraging generative models to understand and defend
  against adversarial examples.
\newblock \emph{arXiv preprint arXiv:1710.10766}, October 2017.

\bibitem[Song et~al.(2018)Song, Shu, Kushman, and Ermon]{song2018constructing}
Yang Song, Rui Shu, Nate Kushman, and Stefano Ermon.
\newblock Constructing unrestricted adversarial examples with generative
  models.
\newblock In \emph{Advances in Neural Information Processing Systems}, pp.\
  8312--8323, 2018.

\bibitem[Szegedy et~al.(2013)Szegedy, Zaremba, Sutskever, Bruna, Erhan,
  Goodfellow, and Fergus]{szegedy2013intriguing}
Christian Szegedy, Wojciech Zaremba, Ilya Sutskever, Joan Bruna, Dumitru Erhan,
  Ian Goodfellow, and Rob Fergus.
\newblock Intriguing properties of neural networks.
\newblock \emph{arXiv preprint arXiv:1312.6199}, 2013.

\bibitem[Theis et~al.(2015)Theis, van~den Oord, and Bethge]{theis2015a}
Lucas Theis, A{\"a}ron van~den Oord, and Matthias Bethge.
\newblock A note on the evaluation of generative models.
\newblock \emph{arXiv preprint arXiv:1511.01844}, November 2015.

\bibitem[van~den Berg et~al.(2018)van~den Berg, Hasenclever, Tomczak, and
  Welling]{berg2018sylvester}
Rianne van~den Berg, Leonard Hasenclever, Jakub~M Tomczak, and Max Welling.
\newblock Sylvester normalizing flows for variational inference.
\newblock \emph{arXiv preprint arXiv:1803.05649}, March 2018.

\bibitem[van~den Oord et~al.(2016)van~den Oord, Kalchbrenner, and
  Kavukcuoglu]{oord2016pixel}
Aaron van~den Oord, Nal Kalchbrenner, and Koray Kavukcuoglu.
\newblock Pixel recurrent neural networks.
\newblock \emph{arXiv preprint arXiv:1601.06759}, January 2016.

\bibitem[Vyas et~al.(2018)Vyas, Jammalamadaka, Zhu, Das, Kaul, and
  Willke]{vyas2018out}
Apoorv Vyas, Nataraj Jammalamadaka, Xia Zhu, Dipankar Das, Bharat Kaul, and
  Theodore~L Willke.
\newblock Out-of-distribution detection using an ensemble of self supervised
  leave-out classifiers.
\newblock In \emph{Proceedings of the European Conference on Computer Vision
  ({ECCV})}, pp.\  550--564. openaccess.thecvf.com, 2018.

\bibitem[White(1982)]{white1982maximum}
Halbert White.
\newblock Maximum likelihood estimation of misspecified models.
\newblock \emph{Econometrica: Journal of the Econometric Society}, pp.\  1--25,
  1982.

\bibitem[Xiao et~al.(2017)Xiao, Rasul, and Vollgraf]{xiao2017fashion}
Han Xiao, Kashif Rasul, and Roland Vollgraf.
\newblock Fashion-mnist: a novel image dataset for benchmarking machine
  learning algorithms.
\newblock \emph{arXiv preprint arXiv:1708.07747}, 2017.

\bibitem[Xu et~al.(2018)Xu, Chen, Zhao, Li, Bu, Li, Liu, Zhao, Pei, Feng,
  et~al.]{xu2018unsupervised}
Haowen Xu, Wenxiao Chen, Nengwen Zhao, Zeyan Li, Jiahao Bu, Zhihan Li, Ying
  Liu, Youjian Zhao, Dan Pei, Yang Feng, et~al.
\newblock Unsupervised anomaly detection via variational auto-encoder for
  seasonal kpis in web applications.
\newblock In \emph{Proceedings of the 2018 World Wide Web Conference on World
  Wide Web}, pp.\  187--196. International World Wide Web Conferences Steering
  Committee, 2018.

\end{thebibliography}
\bibliographystyle{iclr2020_conference}

\newpage
\appendix
\section{Proofs}
\subsection{Proof of Proposition~\ref{prop:normalization}}
\label{app:normalization}
\normalization*
\begin{proof}
Note that in \textit{training mode}, we can write down the batch norm function explicitly:
$$
\text{BatchNorm}(\vz; \gamma, \beta, \epsilon) = \frac{\vz - \frac{1}{b}\sum_{i=1} \vz_i}{\sqrt{\frac{1}{b-1} \sum_{i=1} (\vz_i - \frac{1}{b}\sum_{i=1}\vz)^2 + \epsilon}} \cdot \gamma + \beta
$$
if we fix $\vz_{-i}$, $\gamma$, $\beta$ and $\epsilon$, then $\text{BatchNorm}(\vz; \gamma, \beta, \epsilon)$ is a deterministic function over $\vz_{i}$. The corresponding neural network $g_\vtheta(\vx_i; \vx_{-i})$ is also a deterministic function over $\vx_{i}$ (given $\vx_{-i}$).

We summarize likelihood-based generative models into two types:
\begin{itemize}
    \item In the first type, the neural network produces the parameters for a certain tractable distribution, which is used to compute the likelihood (or its lower bound). This includes VAE and autoregressive models. Since the distribution over $\vx_{i}$ is parametrized by the deterministic function $g_\vtheta$, so the likelihood is naturally normalized.
    \item In the second type, RealNVP models are composed by affine coupling layers~\citep{dinh2014nice}, which contains transformations $(x, y) \mapsto (x', y')$ where
\begin{align}
    x' = x, \quad y' = s(x) \cdot y + t(x)
\end{align}
where $s(\cdot)$ and $t(\cdot)$ are neural networks (potentially with batch normalization layers). This has the reverse
\begin{align}
    x = x', \quad y = (y' - t(x')) / s(x')
\end{align}
Since $\text{BatchNorm}$ is a deterministic function over $\vz_i$ when $\vz_{-i}$ is fixed, the corresponding $s(\cdot)$ and $t(\cdot)$ functions also have the same property. Therefore, the corresponding RealNVP model over $\vx_i$ with fixed $\vx_{-i}$ is also invertible.

As $g_\vtheta$ is a deterministic function over $\vx_i$, the corresponding conditional likelihood is also normalized.
\end{itemize}
This completes the argument for likelihood-based generative models.
\end{proof}

\subsection{Proof of Proposition~\ref{prop:pseudo}}
\label{app:pseudo}
\pseudo*
\begin{proof}
We treat each sample in the batch as a separate random variable, and consider the corresponding graphical model where batch normalization is involved. Since \textit{training mode} aggregates the statistics of all the samples, all the random variables are connected to each other. 

First, we need to show that for large enough $b$, there exists a joint distribution $\tilde{p}_\vtheta(\vx_1, \vx_2, \ldots, \vx_b)$ that is compatible with all the conditional distributions we defined~\citep{arnold1989compatible}. From Theorem 4.1 in~\citep{arnold1989compatible}, we need to show that there exists a function $f$ such that 
\begin{align}
    \frac{\tilde{p}_\vtheta(\vx_i | \vx_{-i})}{\tilde{p}_\vtheta(\vx_j | \vx_{-j})} = \frac{f(\vx_{-j})}{f(\vx_{-i})}
    \label{eq:compatible-condition}
\end{align}
Let $f(\vx_{-j}) = \bb{E}_{\vx_j} [\sum_{k \neq j} \log \tilde{p}_\vtheta(\vx_{k} | \vx_{-k})]$ be the sum of conditional log-likelihood of $\vx_{k}$ under the expectation of some distribution over $\vx_{j}$. For large enough $b$, the value of $\vx_{j}$ will not affect the batch statistics, so $f(\vx_{-j}) = \sum_{k \neq j} \log \tilde{p}_\vtheta(\vx_{k} | \vx_{-k})$, so the condition for compatibility in Equation~\ref{eq:compatible-condition} holds, and there exists a compatible joint distribution, which we denote as $\tilde{p}_\vtheta(\vx_1, \vx_2, \ldots, \vx_b)$.

Next, we verify the conditions under which the product of the conditional likelihoods is a pseudo-likelihood of $\tilde{p}_\vtheta(\vx_1, \vx_2, \ldots, \vx_b)$.
The conditional likelihood for any sample $\vx_{j}$ is independent of $\vx_{-j}$ conditioned on its neighbors (which are $\vx_{-j}$), so 
$$\gL_{\text{train}}(\{x_j\}_{j=1}^{b}; \vtheta) = \sum_{j=1}^{b} \log \tilde{p}_\vtheta(\vx_j \vert \vx_{-j})$$ is the pseudo-log-likelihood objective~\citep{besag1975statistical} for $\{x_j\}_{j=1}^{b}$, which approximates the log-likelihood of the joint distribution.
\end{proof}

\section{Fooling likelihood-based permutation tests}
\label{sec:permutation}

As suggested by empirical evidence, 
we assume that model mis-specification always occurs and that our likelihood estimates cannot be used naively for detecting out-of-distribution samples 
due to the mode seeking nature of $\KL$.
Under this assumption, likelihood-based permutation tests may seems more effective than 
simple thresholding rules for
detecting out-of-distribution samples, as it treats both high-likelihood and low-likelihood samples as OoD. 
In this section, however, we show that likelihood-based permutation test can also be easily fooled, making them infeasible for OoD detection in a straightforward manner. 

We consider fooling a model $p_\vtheta(\vx)$ using samples generated from another model $q_\vphi(\vx)$ (trained on another dataset), where we modify the ``temperature'' $T$ of the generative model~\citep{kingma2018glow} (see Appendix~\ref{app:temperature} for details). Then we gather the likelihood estimates for samples in $p(\vx)$ and $q(\vx)$ and obtain the area under the ROC curve (AUC) using $T_{\text{perm}}$ as prediction values to classify if $\vx$ is OoD. We consider a RealNVP~\citep{dinh2016density} model trained on CIFAR and a PixelCNN++~\citep{salimans2017pixelcnn} model trained on Fashion MNIST. We use a RealNVP trained on SVHN, and a PixelCNN++ trained on MNIST to fool the two $p_\vtheta(\vx)$ models respectively. For RealNVP we set $T = 1.10$ and for PixelCNN++ we set $T = 1.17$.

In Figure~\ref{fig:roc-curve}, we show that it is possible to produce a distribution $q_\vphi(\vx)$ with generative models such that its likelihood distribution evaluated by $p_\vtheta(\vx)$ lies around the median likelihood in $p(\vx)$. This results in low AUC and thereby confusing the test based on $T_{\text{perm}}$. We show some out-of-distribution sample with similar bits-per-dimension (BPD) to training samples in Figure~\ref{fig:perm-samples}, where the generated samples look drastically different from training samples visually, even though they have similar log-likelihoods. We note that flipping the classifier predictions (i.e. using $-T_{\text{perm}}$) will not address this issue, since we would then treat low likelihood samples as more likely to be in-distribution.


\begin{figure}
\centering
\begin{subfigure}{0.40\textwidth}
    \centering
    \includegraphics[width=\textwidth]{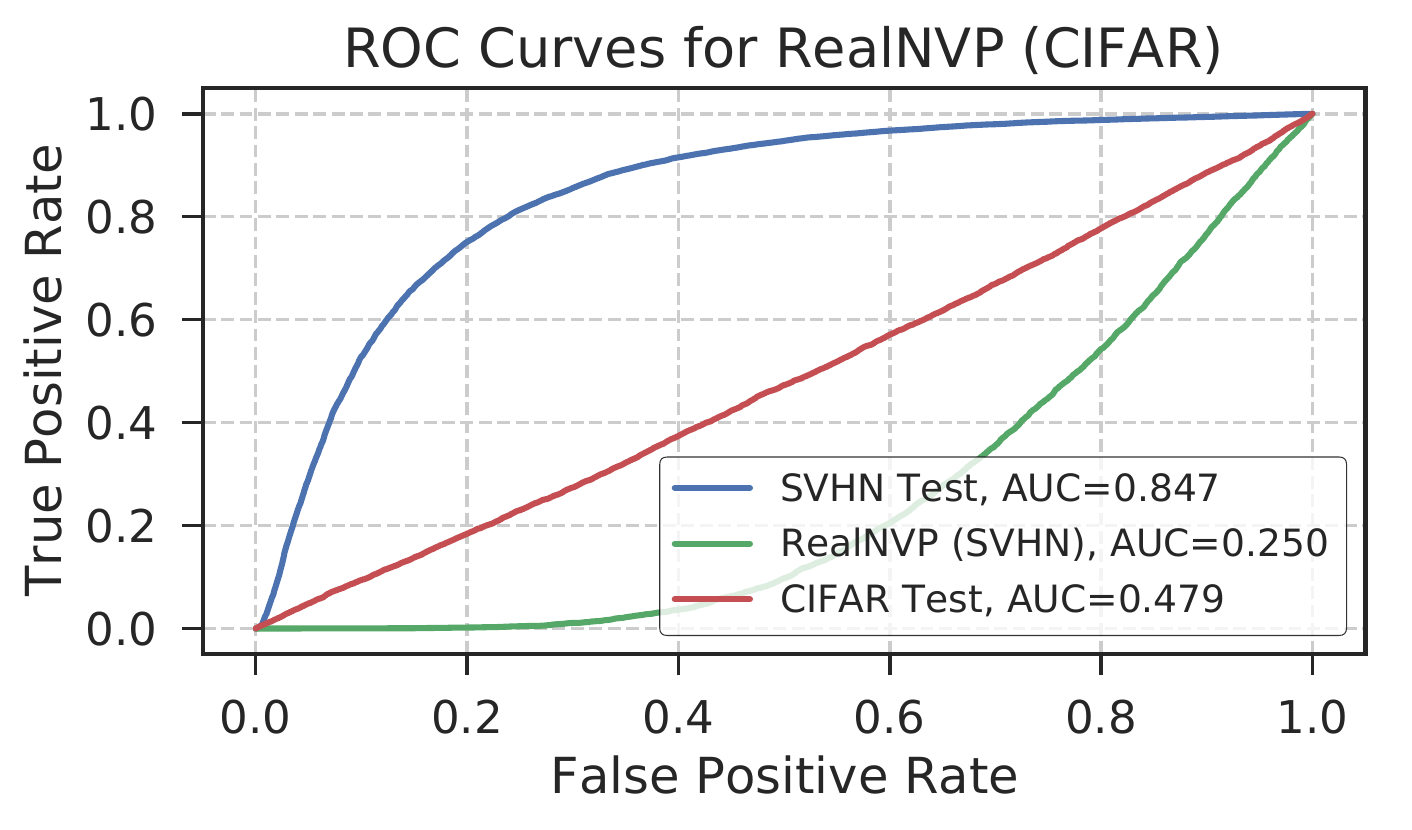}
\end{subfigure}
~
\begin{subfigure}{0.40\textwidth}
    \centering
    \includegraphics[width=\textwidth]{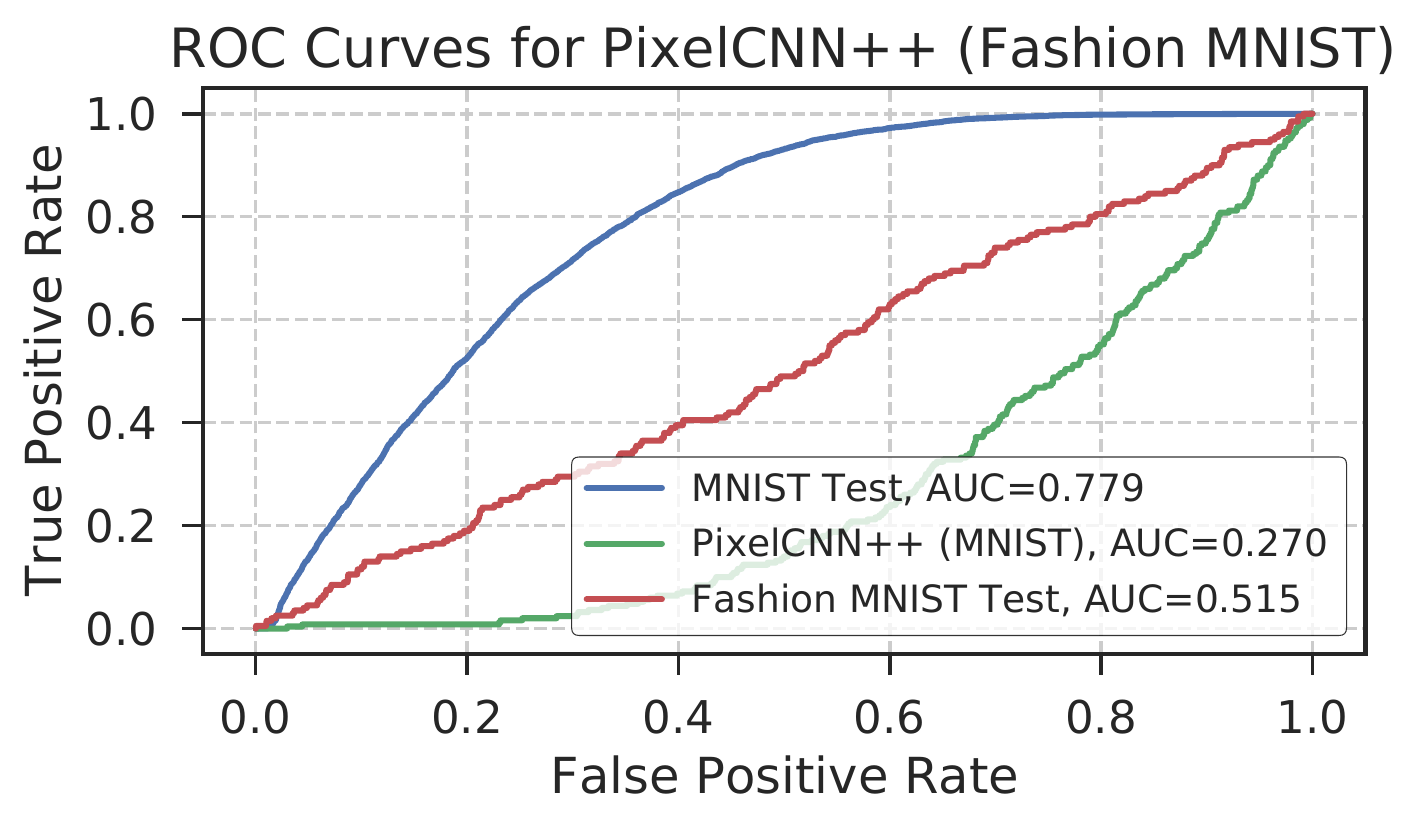}
\end{subfigure}
\caption{ROC curves for using the $p$-values for likelihood-based permutation tests. We assign positive labels to samples in $q(\vx)$ and negative labels to samples in $p(\vx)$. While such tests can detect OoD samples from other datasets, they could be confused by samples from another generative model.}
\label{fig:roc-curve}
\end{figure}

\begin{figure}
\centering
\begin{subfigure}{0.48\textwidth}
    \centering
    \includegraphics[width=\textwidth]{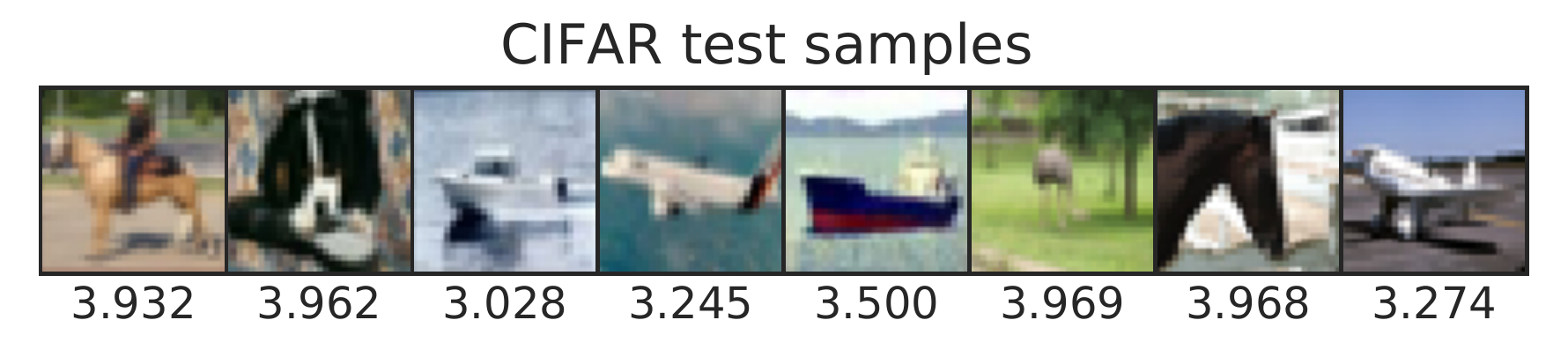}
\end{subfigure}
~
\begin{subfigure}{0.48\textwidth}
    \centering
    \includegraphics[width=\textwidth]{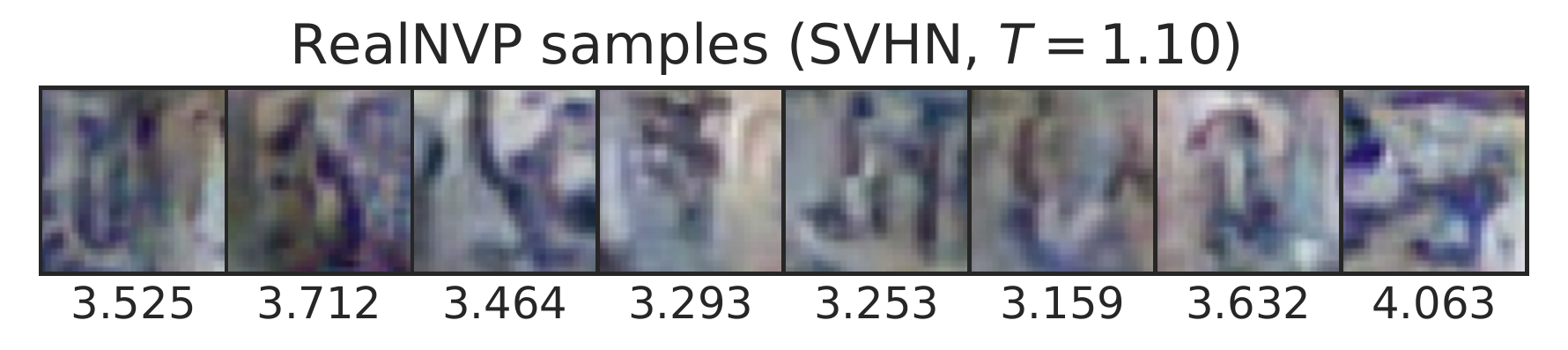}
\end{subfigure}
~
\centering
\begin{subfigure}{0.48\textwidth}
    \centering
    \includegraphics[width=\textwidth]{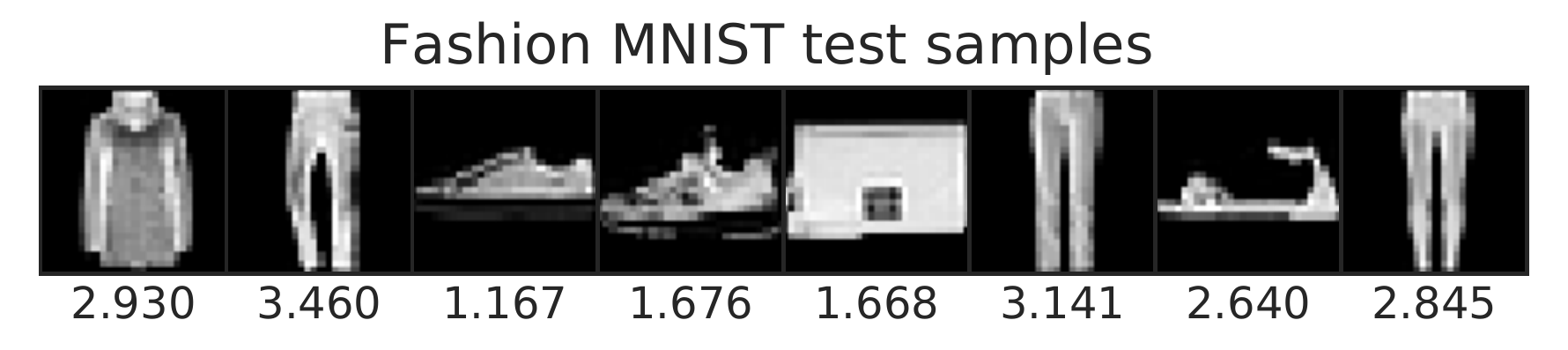}
\end{subfigure}
~
\begin{subfigure}{0.48\textwidth}
    \centering
    \includegraphics[width=\textwidth]{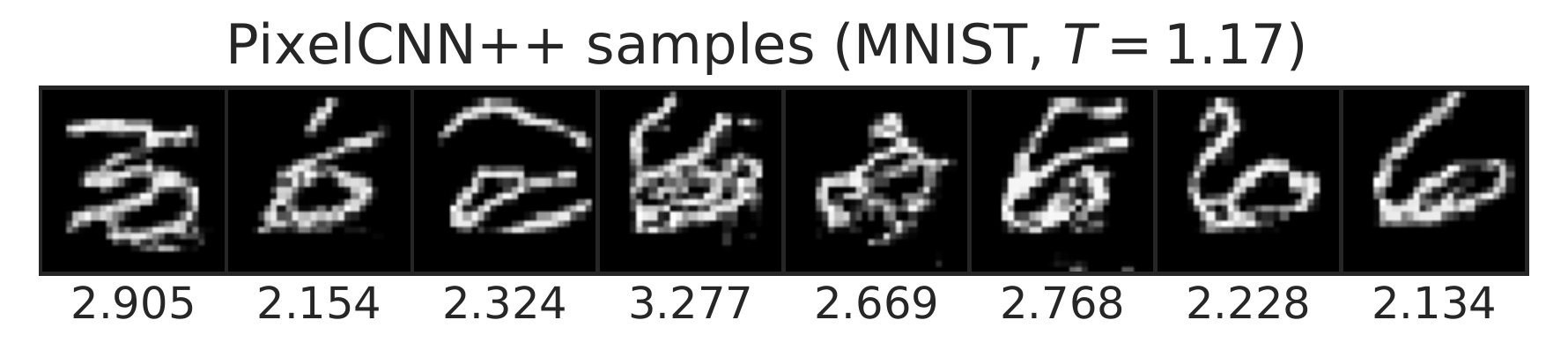}
\end{subfigure}
\caption{Samples and their BPD evaluated under $p_\vtheta(\vx)$. (Top) $p_\vtheta(\vx)$ is a RealNVP trained on CIFAR. (Bottom) $p_\vtheta(\vx)$ is a PixelCNN++ trained on Fashion MNIST.}
\label{fig:perm-samples}
\end{figure}



\section{Additional Experimental Details}
\subsection{Model Architectures}
For RealNVP~\citep{dinh2016density}, we consider downscale once, with 4 residual blocks for each affine coupling layer and 32 channels for each convolutional layer, except for CIFAR where we have 8 residual blocks with 64 channels for each convolutional layer.

For VAE~\citep{kingma2013auto,berg2018sylvester}, our inference network has 6 convolutional layers followed by a fully connected layer, and the generator network has 8 deconvolutional layers. We add batch normalization to all the convolutional layers and deconvolutional layers, except for the last two deconvolutional layers. We detail the convolutional layer hyperparameters in Table~\ref{tab:vae-arch}.

\begin{table}[h]
    \centering
    \caption{Architecture for VAE convolutional layers. $k$ denotes image width divided by 4, $c$ denotes the number of channels of the image (1 for gray, 3 for colored).}
    \label{tab:vae-arch}
    \begin{tabular}{c|cccc}
    \toprule
    \multicolumn{5}{c}{Inference network} \\\midrule
       Layer  & Channels & Kernel & Stride & Padding \\
    1     &  $32c$ & 5 & 1 & 2 \\
    2     &  $32c$ & 5 & 2 & 2 \\
    3     &  $32c$ & 5 & 1 & 2 \\
    4     &  $64c$ & 5 & 2 & 2 \\
    5     &  $64c$ & 5 & 1 & 2 \\
    6     &  256 & $k$ & 1 & 0 \\\midrule
    \multicolumn{5}{c}{Generator network} \\\midrule
       Layer  & Channels & Kernel & Stride & Padding \\
    1     &  $64c$ & 5 & 1 & 2 \\
    2     &  $64c$ & 5 & 2 & 2 \\
    3     &  $64c$ & 5 & 1 & 2 \\
    4     &  $32c$ & 5 & 2 & 2 \\
    5     &  $32c$ & 5 & 1 & 2 \\
    6     &  $32c$ & 5 & 1 & 2 \\
    7     &  256   & 5 &  1 & 2 \\
    8     &  $256c$ & 1 & 1 & 0 \\
    \bottomrule
    \end{tabular}
\end{table}

For PixelCNN++, we made two simple modifications to default hyperparameters~\citep{salimans2017pixelcnn}. We reduce the number of filters in each ResNet block from 160 to 80, and we use batch normalization after convolution as opposed to adding weight normalization over convolution.

All the models are trained with default optimizer hyperparameters with a batch size of 64. We use a batch size of 64 also in evaluation mode.

\subsection{Experimental procedures for Section~\ref{sec:permutation}}
\label{app:temperature}
We consider controlling the entropy of the generated samples by controlling a temperature hyperparameter $T$. In RealNVP, this is realized by multiplying the latent variable by a factor of $T$, which is equivalent to sampling from $\gN(0, T^2 \mI)$. In PixelCNN++, this is realized by dividing each pre-softmax scalar output by $T$; larger $T$ values would lead to higher entropy samples.

\subsection{Likelihood differences between training mode and testing mode on generated data}
\label{app:batchnorm-add}

We consider measuring the likelihood differences between \textit{training mode} and \textit{testing mode} using a RealNVP trained on CIFAR. The samples are generated from a RealNVP trained on SVHN, with several temperatures $T \in \{0.7, 1.0, 1.3\}$. Similar to the observation with SVHN samples, there is a large gap between \textit{training mode} and \textit{evaluation mode} evaluations.

\begin{table}[h]
\centering
\caption{Log-likelihood (measured in bits per dimension) calculated with a RealNVP trained on CIFAR and evaluated on generated samples from a RealNVP trained on SVHN. We report likelihood results with \textit{training mode}, \textit{evaluation mode}, and their difference ($\Delta$).}
\begin{tabular}{l|cc|c}
\toprule
    Dataset       & Mode  & BPD & $\Delta$ \\\midrule
   \multirow{2}{*}{CIFAR}  & \textit{evaluation} & 3.48 & \multirow{2}{*}{0.03} \\
     & \textit{training} & 3.51 &      \\\midrule
    \multirow{2}{*}{SVHN}         & \textit{evaluation} & 2.44 & \multirow{2}{*}{8.56} \\
             & \textit{training} & 11.10 &      \\\midrule
    \multirow{2}{*}{RealNVP (SVHN, $T = 1.0$)}         & \textit{evaluation} & 2.88 & \multirow{2}{*}{325} \\
         & \textit{training} & 328 &      \\\midrule
    \multirow{2}{*}{RealNVP (SVHN, $T = 0.7$)}         & \textit{evaluation} & 1.51 & \multirow{2}{*}{315} \\
         & \textit{training} & 317 &      \\\midrule
    \multirow{2}{*}{RealNVP (SVHN, $T = 1.3$)}         & \textit{evaluation} & 5.43 & \multirow{2}{*}{10.57} \\ & \textit{training} & 16.0 &      \\\bottomrule
    
\end{tabular}
\end{table}

\subsection{Out-of-distribution detection with alternative values of $r_1$ and $r_2$}
\label{app:r1r2}
We include additional results with alternative values of $r_1$ and $r_2$. If our method is able to achieve high performance with small $r_2$, this suggests that we can detect the OoD examples realiably even as they occupy a small portion within the batch. We consider $r_1 \to 0$ (which is \textit{evaluation mode}) and $r_2 \in \{0.15, 0.3, 0.5, 0.9\}$; we show our results in Table~\ref{tab:batchperm-exp-app}.

\begin{table}[h]
    \centering
    \caption{Out-of-distribution classification evaluated with AUC (left) and Average Precision (right). Rotation denotes $q(\vx)$ uses images in $p(\vx)$ yet randomly rotate each image by $d \in (90, 270)$ degrees.}
    
    \begin{tabular}{ccc|cccc}
    \toprule
    $p(\vx)$ & Model & $q(\vx)$ & $r_2 = 0.15$ & $r_2 = 0.3$ & $r_2 = 0.5$ & $r_2 = 0.9$ \\\midrule
    \multirow{4}{*}{\makecell{Fashion \\ MNIST}} & \multirow{4}{*}{RealNVP} 
        & Rotation & 0.93 / 0.95 & 0.91 / 0.94 &  0.98 / \textbf{0.99} & \textbf{0.99} / \textbf{0.99} \\
     &  & MNIST    & 0.92 / 0.93 & 1.00 / 1.00  & 1.00 / 1.00  & \textbf{1.00} / \textbf{1.00} \\
     &  & Omniglot & 0.96 / 0.97 & 0.95 / 0.96 & 1.00 / 1.00 & \textbf{1.00} / \textbf{1.00} \\
     &  & KMNIST   & 0.86 / 0.88 & 1.00 / 1.00 & 1.00 / 1.00 & \textbf{1.00} / \textbf{1.00} \\\midrule
    \multirow{4}{*}{\makecell{Fashion \\ MNIST}} & \multirow{4}{*}{VAE}
        & Rotation & 0.88 / 0.91 & 0.94 / 0.95 & 0.94 / 0.95 & \textbf{0.97} / \textbf{0.98} \\
     &  & MNIST    & 0.86 / 0.88 & 0.99 / 0.99 & 0.56 / 0.64 & \textbf{1.00} / \textbf{1.00} \\
     &  & Omniglot & 0.93 / 0.95 & 0.98 / 0.99 & 0.90 / 0.83 & \textbf{1.00} / \textbf{1.00} \\
     &  & KMNIST   & 0.85 / 0.88 & 0.99 / 0.99 & 0.84 / 0.87 & \textbf{1.00} / \textbf{1.00} \\\midrule
     \multirow{4}{*}{CIFAR} & \multirow{4}{*}{RealNVP} 
        & Rotation & 1.00 / 1.00 & 1.00 / 1.00 & 1.00 / 1.00 & \textbf{1.00} / \textbf{1.00} \\
     &  & SVHN     & 0.92 / 0.90 & 1.00 / 1.00 & 1.00 / 1.00 & \textbf{1.00} / \textbf{1.00}  \\
     &  & ImageNet & 0.74 / 0.72 & 0.92 / 0.90 & 0.98 / 0.98 & \textbf{0.98} / \textbf{0.97} \\
     &  & LSUN     & 0.77 / 0.75 & 0.98 / 0.97 & 1.00 / 1.00 & \textbf{0.99} / \textbf{0.98} \\\bottomrule
     \multirow{4}{*}{CIFAR} & \multirow{4}{*}{PixelCNN++}
        & Rotation & 0.87 / 0.83 & 0.97 / 0.94 & 1.00 / 0.99 & \textbf{0.99} / \textbf{0.99} \\
     &  & SVHN     & 0.76 / 0.72 & 0.95 / 0.92 & 1.00 / 1.00 & \textbf{0.99} / \textbf{0.99} \\
     &  & ImageNet & 0.62 / 0.58 & 0.74 / 0.70 & 0.87 / 0.85 & \textbf{0.89} / \textbf{0.87} \\
     &  & LSUN     & 0.72 / 0.69 & 0.80 / 0.75 & 0.98 / 0.98 & \textbf{0.98} / \textbf{0.97} \\\bottomrule
    \end{tabular}
    \label{tab:batchperm-exp-app}
\end{table}

\section{Details for the derivation in Section~\ref{sec:example}}
\label{app:example}

Suppose we try to learn some distribution $p(\vx)$ with a 2-d flow model with one coupling layer~\citep{dinh2016density}: $f: (x_1, x_2) \mapsto (z_1, z_2)$ where
\begin{align}
z_1 &= x_1 \\
z_2 &= x_2 + (x_1 - E(x_1)) / \sqrt{V(x_1)} \cdot \gamma + \beta
\end{align}
with learnable parameters $\gamma, \beta$; $E(x_1)$, $V(x_1)$ are the empirical mean and variance of a batch of $x_1$ in \textit{training mode}, and the mean and variance of $p(x_1)$ in \textit{evaluation mode}. 

We assume the distribution $p(x_1)$ to be a even mixture of two Gaussians ($\gN(-1, \sigma^2)$ and $\gN(1, \sigma^2)$ with $\sigma \to 0$, so it concentrates around $-1$ and $1$), and $p(x_2) = \delta_0(x_2)$ ($x_2$ concentrates at zero), so $z_2 \approx (x_1 - E(x_1)) / \sqrt{V(x_1)} \cdot \gamma + \beta$. 

The Jacobian of $f$ is:
\begin{align}
    J = \frac{\partial f(\vx)}{\partial \vx} = \begin{bmatrix}
    1 & 0 \\
    \gamma / \sqrt{V(x_1)} & 1
\end{bmatrix}
\end{align}
so its determinant is one.
The likelihood of the flow model is simply 
$$p_f(x_1, x_2) = p(f(x_1, x_2)) |J| = p(f(x_1, x_2))$$

We assume the prior $p(z_1, z_2)$ is standard Gaussian $\mathcal{N}(0, I)$. Therefore, we have:
\begin{align}
    p_f(x_1, x_2) \approx \varphi(x_1) \cdot \varphi((x_1 - E(x_1)) / \sqrt{V(x_1)} \cdot \gamma + \beta)
\end{align}
where $\varphi(x)$ is the probability density function for the standard Gaussian distribution. Note that $\gamma, \beta$ only depends on the second term. The maximum likelihood solution for $\beta = 0$ (due to symmetry); for $\gamma$, we can treat $p(x_1)$ as two point distributions at $-1$ and $1$, so we have the optimal $\gamma \approx \argmin_s -s^2/2 + \log s = 1$. Moreover, $\bb{E}_p[x_1] = 0$ and $\Var_p[x_1] \approx 1$.

For $p(x_1, x_2)$ in evaluation mode, we have
\begin{align}
    \bb{E}_p[\log p_f(x_1, x_2)] &\approx \bb{E}_p[\log \varphi(x_1) + \log \varphi((x_1 - E(x_1)) / \sqrt{V(x_1)}] \\
    & = \bb{E}_p[\log \varphi(x_1) + \log \varphi(x_1)] \approx -2.84.
\end{align}
For $q(x_1, x_2)$ in evaluation mode, we have
\begin{align}
    \bb{E}_q[\log p_f(x_1, x_2)] \approx \bb{E}_q[\log \varphi(x_1) + \log \varphi(x_1)] \approx -1.92.
\end{align}
For $q(x_1, x_2)$ in training mode, $E(x_1) = 0$, $V(x_1) = 1/\sqrt{12}$, so we have:
\begin{align}
    \bb{E}_q[\log p_f(x_1, x_2)] = \bb{E}_q[\log \varphi(x_1) + \log \varphi(\sqrt{12} x_1)] \approx -2.38.
\end{align}




\end{document}